\documentclass{vldb}

\usepackage{tabularx}
\usepackage{booktabs} 
\usepackage{graphicx}
\usepackage{multirow}
\usepackage{balance}  

\usepackage{algorithm}
\usepackage{algpseudocode}
\usepackage{amsmath}
\usepackage{mathtools}
\usepackage{amsmath}
\usepackage{amssymb}
\usepackage{mathrsfs}

\usepackage{subcaption}

\floatname{algorithm}{\textbf{Algorithm}}

\vldbTitle{Fine-grained Pattern Matching Over Streaming Time Series}
\vldbAuthors{Rong Kang, Chen Wang, Peng Wang, Yuting Ding, Jianmin Wang}
\vldbVolume{11}
\vldbNumber{2}
\vldbYear{2017}
\vldbDOI{https://doi.org/TBD}

\begin{document}
	

	\title{Fine-grained Pattern Matching Over Streaming Time Series}
	
		\author{
		\alignauthor
		Rong Kang$^{\S}$  Chen Wang$^{\S}$ Peng Wang$^{\ddag}$ Yuting Ding$^{\S}$ Jianmin Wang$^{\S}$\\
		\affaddr{$^{\S}$Tsinghua University, China}\\
		\affaddr{\{kr11, dingyt16\}@mails.tsinghua.edu.cn \{wang\_chen, jimwang\}@tsinghua.edu.cn}\\
		\affaddr{$ ^{\ddag} $ Fudan University, China $ \quad $pengwang5@fudan.edu.cn}
}
	\maketitle
	
	\begin{abstract}
		Pattern matching of streaming time series with lower latency under limited computing resource comes to a critical problem, especially as the growth of Industry 4.0 and Industry Internet of Things. However, against traditional single pattern matching problem, a pattern may contain multiple segments representing different statistical properties or physical meanings for more precise and expressive matching in real world.
		Hence, we formulate a new problem, called ``fine-grained pattern matching'', which allows users to specify varied granularities of matching deviation to different segments of a given pattern, and fuzzy regions for adaptive breakpoints determination between consecutive segments. 
		In this paper, we propose a novel two-phase approach. In the pruning phase, we introduce Equal-Length Block (ELB) representation together with Block-Skipping Pruning (BSP) policy, which guarantees low cost feature calculation, effective pruning and no false dismissals. 
		In the post-processing phase, a delta-function is proposed to enable us to
		conduct exact matching in linear complexity. 
		Extensive experiments are conducted to evaluate on synthetic and real-world datasets, which illustrates that our algorithm outperforms the brute-force method and MSM, a multi-step filter mechanism over the multi-scaled representation.
		
	\end{abstract}
	
	\section{Introduction}
\label{sec:introduction}
%

Time series are widely available in diverse application areas, such as Healthcare \cite{wei_atomic_2005}, financial data analysis \cite{wu_online_2004} and sensor network monitoring \cite{zhu_efficient_2003}, and they turn the interests on spanning from developing time series database \cite{jensen_time_2017}.
In recent years, the rampant growth of Industry 4.0 and Industry Internet of Things, especially the development of intelligent control and fault prevention to complex equipment on the edge, urges more challenging demands to process and analyze streaming time series from industrial sensors with low latency under limited computing resource \cite{zhao_adaptive_2014}.

As a typical workload, similarity matching over streaming time series has been widely studied for fault detection, pattern identification and trend prediction, where accuracy and efficiency are the two most important measurements to matching algorithms \cite{bulut_unified_2005, lian_similarity_2007}. Given a single or a set of patterns and a pre-defined threshold, traditional similarity matching algorithms aim to find matched subsequences over incoming streaming time series, between which the distance is less than the threshold. However, in certain scenarios, this single granularity model is not expressive enough to satisfy the similarity measurement requirements.

The data characteristics of time series can change due to external events or internal systematic variance, which is well explored as change point detection problem \cite{aminikhanghahi_survey_2017}. As the target which users intend to detect in streaming time series, the pattern often contains one or more change points. Moreover, the segments split by change points may have differentiated matching granularities. Let us consider the following example.

\begin{figure*}[!htb]
	\centering
	\includegraphics[width=1\textwidth]{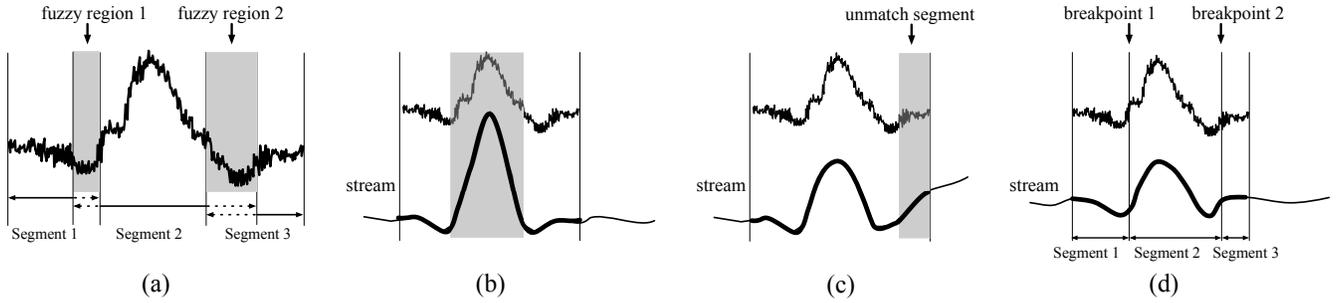}
	\caption{
		(a) The pattern of Extreme Operating Gust(EOG) contains three segments: a slight decrease (Segment 1), followed by a steep rise, a steep drop (Segment 2), and a rise back to the original value (Segment 3).
		(b) A positive result whose Segment 2 has higher amplitude than the pattern.
		(c) A negative result whose Segment 3 is slightly different from the pattern.
		(d) For a positive result, the runtime matching process will bring a clear cut for each segment.
	}
	\label{fig:intro_apply}
\end{figure*}
In the field of wind power generation,
LIDAR systems \cite{schlipf_nonlinear_2013} are able to provide preview information of wind disturbances ahead of wind turbines. For example, Extreme Operating Gust (EOG) \cite{branlard_wind_2009} is a typical gust pattern which is a phenomenon of dramatic changes of wind speed in a short period (1s $ \sim $ 10min). Early detection of EOG can prevent the damage on the turbine  \cite{pace_lidar-based_2012}.
As shown in Figure~\ref{fig:intro_apply}(a), a typical pattern of EOG has three physical phases, and its corresponding shape contains a slight decrease (Segment 1), followed by a steep rise, a steep drop (Segment 2), and a rise back to the original value (Segment 3).
In practice, the Segment 2 is a unique feature of EOG pattern, where the peak is significantly higher than normal value. However, users usually emphasize the shape feature much more than what the exact value is. In other words, the deviation of distance measurement is more tolerant. Yet the matching tolerances of Segment 1 and Segment 3 are stricter to identify the slight decrease and stable trend respectively (after subtracting the level of time series). Figure~\ref{fig:intro_apply}(b) and (c) show positive and negative results respectively. In Figure~\ref{fig:intro_apply}(b), although the segment 2 of the sequence  has much higher amplitude than the pattern, due to the more tolerance of this segment, it's still a positive result. On the contrary, although the sequence in Figure~\ref{fig:intro_apply}(c) is only slightly different from the segment 3 of the pattern, it is determined as a negative result because the matching of Segment 3 is stricter.

This time series pattern has two physical status changes, from Segment 1 to Segment 2 and Segment 2 to Segment 3, but even the experienced expert is hard to locate the exact breakpoint of consecutive segments. In essence, the physical status switch usually associates with a transition period in between, which is obscure to be classified. Hence, it is more reasonable for users to specify a region rather than a single point as the segment boundary (illustrated by two fuzzy regions in Figure~\ref{fig:intro_apply}(a)) when defining the pattern. Within this region, any element is acceptable as the breakpoint, but it is dynamically determined with the incoming streaming time series during the matching process. Figure~\ref{fig:intro_apply}(d) demonstrates a positive result, the runtime matching process will bring a clear cut for each segment.

In summary, above example shows that a pattern may contain multiple segments representing different statistical properties or physical meanings, and users may want to specify different thresholds for matching distance measurement and elastic boundaries to separate the segments. Besides industrial examples, there are similar situations in other fields of streaming time series pattern matching, such as electrocardiogram in Healthcare, technique analysis based program trading in stock market, etc. Therefore, we formulate a new problem, \textit{fine-grained} pattern matching over streaming time series, which allows users to specify varied granularities of matching deviation to different segments of a given pattern, and fuzzy regions for adaptive breakpoint determination between consecutive segments.

Given the multiple segments and multiple thresholds settings, comparing with traditional stream pattern matching problem, the technical challenges are mainly caused by the dynamic runtime breakpoint determination. In the pruning phase, the traditional dimension reduction approaches cannot be simply applied due to varied lengths of segments. Therefore, the first challenge is the effective feature extraction approach to tolerate adaptive segment boundary and guarantee no false dismissals. During the post-processing phase, given quadratic possibilities of begin point and end point in fuzzy regions, an efficient post-processing algorithm is the second challenging job.

For the same reason, this new problem cannot be trivially simplified as a series of single pattern matching and the position alignment of these segments afterwards. Because given the existence of fuzzy breakpoint, the dynamic boundary determination of consecutive segments may cause the change of average deviation as the new streaming data coming in, which results in the ineffectiveness of pre-segmentation. This is analyzed in detail in section~\ref{sec:discussion}.

Although many techniques have been proposed for time series similarity matching, they do not aim to solve the aforementioned challenges. For streaming time series matching, some recent works take advantage of similarity or correlation of multiple patterns and avoid the whole matching of every single patterns~\cite{lian_similarity_2007, wei_atomic_2005}. Similarly, most of previous approaches for subsequence similarity search explore and index the commonalities of time series in database to accelerate the query \cite{loh_subsequence_2004, wang_data-adaptive_2013}. Both of them are not optimized for the fine-grained pattern matching scenario.

In this paper, we propose a novel algorithm to tackle the fine-grained pattern matching challenge. Equal-Length Block (ELB) representation is proposed to divide both the pattern and the stream into equal-length disjoint blocks, and then featurize a pattern block as a bounding box and a window block as a single value, which guarantees low cost feature calculation for newly arrival data and no false dismissals during the pruning phase. In the post-processing phase, we propose a delta-function representing the deviation of distance between pattern and candidate from the pre-defined threshold, which enables us to conduct exact matching and determine adaptive breakpoints in linear complexity. Considering the observation that elements appear in a series of consecutive windows in stream manner matching, we propose Block-Skipping Pruning (BSP) policy to optimize the pruning phase by leveraging the overlapping relation between consecutive sliding windows.

In summary, this paper makes the following contributions:
\begin{itemize}
	\item We define a new fine-grained similarity matching problem to more precisely match streaming time series and given pattern, allowing multiple thresholds to different pattern segments split by adaptive breakpoints.
	\item We propose a novel ELB representation and an adaptive post-processing algorithm which processes exact similarity matching in linear complexity. Furthermore, BSP policy is proposed as an optimization to accelerate the pruning process.
	\item We provide a comprehensive theoretical analysis and carry out sufficient experiments on real-world and synthetic datasets to illustrate that our algorithm outperforms prior-arts.
\end{itemize}

The rest of the paper is arranged as follows:
Section~\ref{sec:problem_overview} formally defines our problem.
Section~\ref{sec:lower_bounding} proposes ELB representation and its two implementations.
Section~\ref{sec:algorithm} introduces a pruning algorithm based on ELB and
presents an adaptive post-processing algorithm.
Section~\ref{sec:smp} introduces the BSP optimization.
Section~\ref{sec:experiment} conducts comprehensive experiments.
Section~\ref{sec:relate_work} gives a brief review of the related work.
Finally, Section~\ref{sec:conclusion} concludes the paper.

	\section{Problem Overview}
\label{sec:problem_overview}

\subsection{Problem Definition}
\newdef{definition}{Definition}
\newtheorem{thm}{Theorem}[section]
Pattern $P$ is a time series which contains $ n $ number of elements $(p_1, \cdots, p_{n})$. We denote the subsequence $ (p_i, \cdots,  p_j)$ of $P$ by $ P[i:j] $. Logically, $P$ could be divided into several consecutive \textit{segment}s which may have varied thresholds of matching deviation.
\begin{definition}\label{def:segment}
	(\textbf{Segment}): Given a pattern $P$, $ P $ is divided into $ b $ number of non-overlapping subsequences in time order, represented as $P_1, P_2,$ $\cdots, P_{b}$, in which the $k$-th subsequence $ P_k $ is defined as the $k $-th segment and associated with a specified threshold $ \varepsilon_k $.
\end{definition}

\begin{definition}\label{def:breakpoint}
	(\textbf{Breakpoint}): Let the $k$-th segment $ P_k $ ($k\in [1,b)$) be equal to $ P[i:j] $($ 1 \leqslant i \leqslant j < n$), then the index $j$ is defined as the $k$-th breakpoint to separate $ P_k $ and $ P_{k+1} $, denoted by $bp_k$.
	
\end{definition}

Due to the difficulty to set the exact breakpoint of two adjacent segments, users are allowed to specify a fuzzy region instead, which actually covers the breakpoint eventually determined by the proposed algorithm and its proximity.
\begin{definition}\label{def:Break Region}
	(\textbf{Break Region}): Given segments $ P_{k} $ and $ P_{k+1} $ ($ k\in[1,b) $), $bp_k$ is allowed to be set to any element in region $ [l_k, r_k]$, which is defined as the $k$-th \textit{break region}, denoted by $ BR_k $. Once $bp_k$ is determined, then $ [l_k, bp_k] \subset P_k$ and $ (bp_k, r_k] \subset P_{k+1}$.
\end{definition}
In this paper, we assume that  adjacent break regions are non-overlapping. Moreover, for consistency, we set $ BR_0 = [l_0,r_0] = [0,0]$.

\begin{definition}\label{def: pattern segmentation}
	(\textbf{Pattern Segmentation}): A set of breakpoints $bp_1,\cdots, bp_{b-1}$ determines one unique combination of segments  of $P$, which is called a \textit{pattern segmentation}.
\end{definition}

As shown in Figure~\ref{fig:lb_window_block}(a), for instance, pattern $ P $ is composed of three segments.
The breakpoint $ bp_1 $ between $ P_1 $ and $ P_2 $ is allowed to vary in $[4, 5] $ and $ bp_2 $ in $[11, 13] $.
Consider one segmentation that $ bp_1 = 4$ and $ bp_2 = 12$, $ P $ is divided into
$ P[1:4]$ , $P[5:12]$ and $P[13:15]$.

\begin{figure}[!htb]
	\centering
	\includegraphics[width=0.34\textwidth]{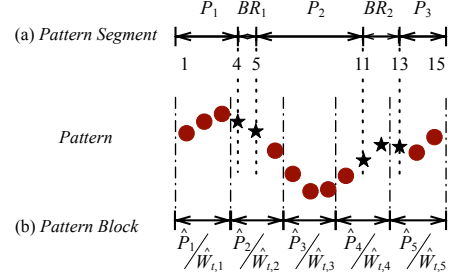}
	\caption{(a) Pattern $ P $ is composed of three segments.
		The breakpoint $ bp_1 $ between $ P_1 $ and $ P_2 $ is allowed to vary in $[4, 5] $ and $ bp_2 $ in $[11, 13] $. \quad
		(b) If we set block size $ w =3 $, thus $ P $ and $ W_t $ are divided into 5 blocks. Note that, the block division is independent of pattern segments.
	}
	\label{fig:lb_window_block}
\end{figure}

A streaming time series $S$ is an ordered sequence of elements  that arrive in time order.
We denote a sliding window on $ S $ which starts with timestamp $ t $ by  $ W_t = (s_{t,1}, s_{t,2},\cdots, \\s_{t, n})$.
According to the segment division of $ P $,
$ W_t $ is also divided into $ b $ segments $W_{t,1}, W_{t,2},\cdots, W_{t,b}$.
We define $ W_t[i:j] $ as the subsequence $(s_{t,i},s_{t,i+1},\cdots s_{t,j}) $.
For convenience, we refer to $ p_i $ and $ s_{t,i} $ as an \textit{element pair}.

There are many distance functions such as \textit{DTW} \cite{berndt_using_1994}, \textit{LCSS} \cite{vlachos_discovering_2002} , $L_p$-norm \cite{yi_fast_2000}, etc. We choose the Euclidean distance which covers a wide range of applications \cite{begum_rare_2014, faloutsos_fast_1994, mueen_online_2010} .
Referring to \cite{mueen_time_2014, mueen_logical-shapelets:_2011}, we use the normalized Euclidean distance to avoid the impact of the variety of segment length in fine-grained pattern matching.
\begin{definition}
	Given two $ n $-length sequences where $X=(x_1,x_2,\cdots,
	x_n)$ and $Y=(y_1,y_2,\cdots,y_n)$,
	we define the normalized Euclidean distance between $X$ and $Y$ as follows:
	\begin{displaymath}
	ED_{norm}(X,Y) = \sqrt{\frac{1}{n}(\sum^{n}_{i=1} |{x_i-y_i}|^2)} 
	\label{eq:ald}
	\end{displaymath}
	\label{def:ld_p}
\end{definition}
Since $ ED_{norm} $ we used is a distance function between two equal-length sequences, there are $ |W_t| = n $ and $ |W_{t,k}| = |P_k| $ for $ k \in [1,b] $.
In addition, we denote by $ ED_{norm}[i:j] $ the normalized Euclidean distance  between $ P[i:j] $ and $ W_t[i:j] $.

\noindent \textbf{Problem Statement}:
	Given a pattern $P$, users pre-define a list of break regions $ BR_1, BR_2, \cdots, BR_{b-1}  $, which divide $ P $ into $ b $ number of segments $ P_1, P_2, \cdots, P_b $, associated with specified thresholds $ \varepsilon_1, \varepsilon_2, \cdots, \varepsilon_b$. 
	For a stream $ S $, the fine-grained similarity matching problem
	is to find all sliding windows $ W_t $ on $ S $, where
	$ W_t $ has at least one segmentation so that, for any $ k \in [1,b] $, it holds that $ED_{norm}(P_k,W_{t,k}) \leqslant \varepsilon_k$ (denoted by $ W_{t,k} \prec P_k$).

	\subsection{Problem  Discussion}
\label{sec:discussion}
Before introducing our approach,
we first analyze two schemes which intend to eliminate the impact of fuzzy breakpoints by transforming a fine-grained pattern to a series of disjoint segments with fixed boundaries. Then we demonstrate that both of these two schemes cannot guarantee no false dismissals.

\noindent \textbf{Scheme 1}:  Given a fine-grained pattern $ P $, we transform it to $ b $ number of disjoint patterns $ P'_1, P'_2, \cdots, P'_b$, where $ \varepsilon'_k = \varepsilon_k (1 \leqslant k \leqslant b) $. $ bp'_k $ is defined as follows:
\begin{equation*}
\begin{split}
bp'_{k} &=
\begin{cases}
l_k &    \varepsilon_k \leqslant \varepsilon_{k+1}\\
r_{k} &otherwise\\
\end{cases}
\end{split}\label{eq:shceme-1}
\end{equation*}

\noindent \textbf{Scheme 2}:  Given a fine-grained pattern $ P $, we transform it to $ 2b-1 $ number of disjoint segments $ P'_1, P'_2, \cdots, P'_{2b-1} $:
\begin{equation*}
\begin{split}
\begin{cases}
P'_{2k-1} = P[r_{k-1}+1: l_k],  \quad \varepsilon'_{2k-1} =  \varepsilon_k &   1 \leqslant k \leqslant b \\
P'_{2k} = P[l_k+1: r_k], \quad \varepsilon'_{2k} =  \max (\varepsilon_k, \varepsilon_{k+1}) &  1 \leqslant k < b \\
\end{cases}
\end{split}\label{eq:shceme-2}
\end{equation*}

We demonstrate that these two schemes cannot guarantee no false dismissals with an example in Figure~\ref{fig:07discussion}.
The pattern $ P $ contains two segments, $ P_1 $ with $ \varepsilon_1  = 4 $ and $P_2$ with $ \varepsilon_2  = 5  $ where $ BR_1 = [4:6] $.
The top of Figure~\ref{fig:07discussion} shows  the element-pair distances between $ P $ and a candidate $ C $.
Setting $ bp_1 = 6 $ makes $ C $ be a fine-grained matching result of $ P $.
Scheme 1 simply assigns elements in $ BR_k $ to its adjacent segment with a larger threshold.
In the example of Figure~\ref{fig:07discussion}, Scheme 1 divides $ P $ into $ P'_1 = P[1:4] $ and $ P'_2 = P[5:9] $, and then falsely discards $ C $ for $ ED_{norm}(P'_2, C'_2 ) > 5 $.
Scheme 2 makes the break region $ BR_k $ be a new segment whose threshold is $ \max (\varepsilon_k, \varepsilon_{k+1}) $.
In Figure~\ref{fig:07discussion}, Scheme 2 divides $ P $ into three segments: $ P'_1 = P[1:4], P'_2 = P[5:6], P'_3 = P[7:9] $, and falsely discards $ C $ for $ ED_{norm}(P'_2, C'_2) > 5 $.
\begin{figure}[!htb]
	\centering
	\includegraphics[width=0.47\textwidth]{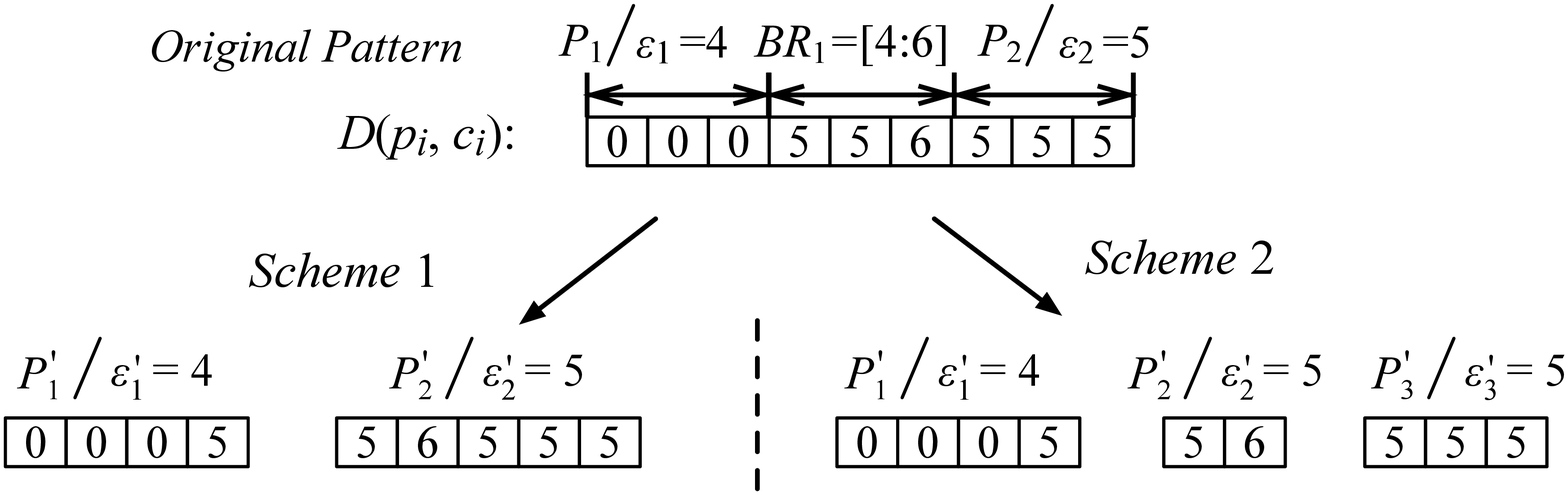}
	\caption{\small Illustration of scheme 1 and 2
	}
	\label{fig:07discussion}
\end{figure}

	\section{Equal-Length Block}\label{sec:lower_bounding}

To avoid false dismissals, a naive method is to slide the window over stream by one element and calculates the corresponding distance, which is computationally expensive. However, one interesting observation is that in most real-world applications, the majority of adjacent subsequences of time series might be similar. This heuristic gives us the opportunity to process multiple consecutive windows together. Based on this hint, we propose a novel representation for both pattern and stream, named as Equal-Length Block (ELB), and the corresponding \textit{lower bounding property} of ELB representation. 

ELB divides the pattern $P$ and the sliding window $ W_t $ into several disjoint $ w $-length blocks while the last indivisible part can be safely discarded.
The number of blocks is denoted by $N = \lfloor n/w \rfloor$.
Based on blocks, $P$ and $W_t$ can be split into $\hat{P} = \{\hat{P}_1, \cdots, \hat{P}_{N}\}$ and $ \hat{W}_t = \{\hat{W}_{t, 1}, \cdots, \hat{W}_{t,N}\} $ respectively, where $\hat{P}_j $ (or $\hat{W}_{t,j}$) is $j$-th block of $P$ (or $W_t$), that is, $\hat{P}_j =\{p_{(j-1)\cdot w+1},p_{(j-1)\cdot w+2},\cdots, p_{j\cdot w}\} $ and $ \hat{W}_{t,j} $ is similar.
Note that, the block division is independent of pattern segments. In other words, a block may overlap with two adjacent segments, and a segment contains more than one block. 
As shown in Figure~\ref{fig:lb_window_block}(b),
we set $ w = 3$ , thus $P$ and $W_t$ are divided into 5 blocks.

In ELB representation, each pattern block $ \hat{P}_j $ is represented by a pair of bounds, upper and lower bounds, which are denoted by $ \hat{P}^u_j $ and $ \hat{P}^l_j $ respectively. Each window block $ \hat{W}_{t, j} $ is represented by a feature value, denoted by $ \hat{W}_{t, j}^f$. It is worth noting that the ELB representation is only an abstract format description, which doesn't specify how to compute upper and lower bounds of $\hat{P}_j$ and the feature of window $ \hat{W}_{t, j}$. Instead, the specific implementation is called as \textit{ELB implementation}.
We can design any ELB implementation, which just needs to satisfy the following lower bounding property:

\begin{definition}\label{def:lower_bounding_property}
	(\textbf{Lower Bounding Property}):
	given $ \hat{P} $ and $ \hat{W}_t $,
	if $\exists \ i \in [0,w)$, $W_{t+i}$ is a fine-grained matching result of $P$,
	then $\forall j \in [1,N]$,
	$ \hat{P}^l_j \leqslant \hat{W}_{t, j}^f \leqslant \hat{P}_j^u$ (marked as $\hat{W}_{t, j} \prec \hat{P}_j$).
\end{definition}


Before introducing our ELB implementation, we first describe how the pattern matching algorithm works based on lower bounding property. Instead of processing sliding windows one-by-one, lower bounding property enables us to process $ w $ consecutive windows together  in the pruning phase.
Given $ N $ number of window blocks $ \{\hat{W}_{t, 1}, \cdots, \hat{W}_{t, N}\} $, if anyone in them (e.g. $\hat{W}_{t,j}$)  doesn't match its aligned pattern block ($\hat{P}_{j}$ correspondingly),  we could skip $w$ consecutive windows, $ W_t, W_{t+1},\cdots, W_{t+w-1} $, together. Otherwise, the algorithm takes these $w$ windows as candidates and uses a post-processing algorithm to examine them one by one.
The lower bounding property enables us to extend the sliding step to $ w $ while guaranteeing no false dismissals. The key challenge is how to design ELB implementation which is both computation efficient and effective to prune sliding windows. In this paper, we propose two ELB implementations, $ELB_{ele}$ and $ELB_{seq}$, which will be introduced in the next sections in turn.

\subsection{Element-based ELB Representation} 
\label{sub:lb_element}
In this section, we present the first ELB implementation, $ELB_{ele}$. The basic idea is as follows. According to our problem statement, if window $W_t $ matches $P$, there exists at least one segmentation in which,  for any segments $W_{t,k}$ and $P_k$ with boundaries $(bp_{k-1},bp_{k}]$, their normalized Euclidean distance is not larger than $ \varepsilon_k $, thus their Euclidean distance holds that:
\[
ED(W_{t,k},P_k) \leqslant \varepsilon_k \sqrt{bp_{k}-bp_{k-1}}
\]
It can be easily inferred that any element pair, $ (p_i, s_{t,i}) $, in the $ k$-th segment ($bp_{k-1} < i\leqslant bp_{k}$) satisfies that:
\begin{equation}
|s_{t,i}-p_i| \leqslant \varepsilon_k \sqrt{bp_{k}-bp_{k-1}}
\label{eq:rationale}
\end{equation}
In other words, if $s_{t,i}$ falls out of the range
\[
	[ \
	p_i - \varepsilon_k \sqrt{bp_{k}-bp_{k-1}},
	 p_i + \varepsilon_k \sqrt{bp_{k}-bp_{k-1}}
	 \ ]
\]
$W_t$ cannot match $P$ under this segmentation.

To extend this observation to an ELB implementation, we need to solve two problems: 1) how to handle the variable breakpoint and 2) how to process $w$ consecutive sliding windows together.

For the first problem, we construct a lower/upper envelope for pattern $P$, $\{[L_1,U_1], \cdots, [L_n,U_n]\}$, which guarantees that if $s_{t,i}$ falls out of $[L_{i},U_i]$, $W_t$ cannot match $P$ for any possible segmentation.
Since $bp_{k-1}\geqslant l_{k-1}$ and  $bp_{k}\leqslant r_{k}$,  we transform Equation~\ref{eq:rationale}  to
\begin{equation}
|s_{t,i}-p_i| \leqslant \varepsilon_k \sqrt{bp_{k}-bp_{k-1}} \leqslant \varepsilon_k \sqrt{r_{k}-l_{k-1}}
\label{eq:max-distance}
\end{equation}
Equation~\ref{eq:max-distance} will be suitable for all possible the $k$-th segments. The right term of Equation~\ref{eq:max-distance} represents the maximal distance between an element pair which belongs to the $ k $-th segments (denoted by $md(k)$).
If $ p_i $ can only belong to the $k$-th segment, we construct the envelope of $p_i$ as $[p_i-md(k), p_i+md(k)]$.
However, for $p_i$ within $ (l_k,r_k]$, it can belong to either the $k$-th segment or $(k+1)$-th segment. In this case, we take the larger one of $md(k)$ and $md(k+1)$ to construct the envelope. Formally, we define a notation $\theta_{ele}(i)$ as follows:
\begin{equation}
\begin{split}
\theta_{ele}(i) &=
\begin{cases}
\max(md(k),md(k+1)),&  i \in (l_k, r_k] \\
md(k),&  i \in (r_{k-1}, l_{k}]
\end{cases}\\
\end{split}\label{eq:threshold}
\end{equation}
and construct the envelope for pattern $P$ based on it:
\begin{equation}
\begin{split}
U_{i} = p_i+\theta_{ele}(i) \\
L_{i} = p_i-\theta_{ele}(i)
\end{split}
\label{eq:element_ul}
\end{equation}

\begin{figure}[!htb]
	\centering
	\includegraphics[width=0.47\textwidth]{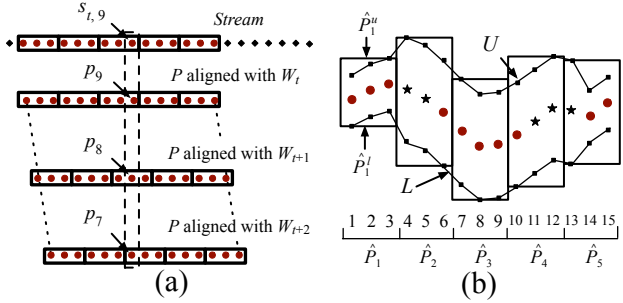}
	\caption{
		(a) The element $ s_{t,9} $ aligns with $ p_9, p_8$ and $p_7 $ at $W_{t},W_{t+1}$ and $W_{t+2}$ respectively.
		(b) $ \hat{P}_j^u $ and $ \hat{P}_j^l $ are constructed by $ U $ and $ L $.
	}
	\label{fig:03elb-ele}
\end{figure}

Now we consider the second problem, how to turn the upper/lower envelope into the bounds of pattern block and how to calculate the feature of window block, so that we could prune $ w $ number of consecutive windows together. We show the basic idea with an example in Figure~\ref{fig:03elb-ele}(a). Assume $w=3$ and $ N = 5 $.
At $W_{t}$ ($t=i\cdot w+1, i = 0, 1, 2, \cdots $),
element $s_{t,9}$ aligns with $ p_{9} $. Similarly, in $W_{t+1}$ (or $W_{t+2}$), $s_{t,9}$ aligns with $p_{8} $ (or $ p_{7} $).
Obviously, if $s_{t,9}$ falls out of all upper and lower bounds of
$ p_{9} $, $ p_{8} $ and $ p_{7} $, these 3 corresponding windows can be pruned together. Note that $s_{t,9}$ is the last element of block $\hat{W}_{t,3}$, and only in this case, all three elements of $P$ aligning with $s_{t,9}$ belong to the same pattern block, $\hat{P}_3$.
Based on this observation, we define $ \hat{P}_j^u $, $ \hat{P}_j^l $ and $ \hat{W}_j^f $ as follows:
\begin{equation}
\begin{cases}
\begin{split}
&\hat{P}_{j}^u = \max \limits_{0 \leqslant i < w} (U_{j\cdot w - i})\\
&\hat{P}_{j}^l = \min \limits_{0 \leqslant i < w} (L_{j\cdot w - i})\\
&\hat{W}_{t,j}^f = last(\hat{W}_{t,j}) = s_{t,j\cdot w}
\end{split}
\end{cases}
\label{eq:element_bound}
\end{equation}
As shown in Figure~\ref{fig:03elb-ele}(b), for each pattern block,
its upper and lower bounds are set to the maximum and minimum of its two envelope lines respectively.
We provide the proof that $ ELB_{ele} $ satisfies the lower bounding property as follows.
Note that, although we focus on the Euclidean distance, $ ELB_{ele} $ can be extended to cases of other $ L_p $-Norm distance measures easily.

\begin{proof}
	Given $P$ and $\hat{W}_t$, suppose that $W_{t,i}$ is a fine-grained matching of $P$ where $0 \leqslant i < w$. 
	Consider any window block $\hat{W}_{t,j}$ where $j\in [1,N]$.
	The last element of $\hat{W}_{t,j}$ is $s_{t,j\cdot w}$
	which aligns with $p_{j\cdot w-i}$.
	Referring to the definitions of $ \delta_{ele} $, it holds that:
	\begin{equation*}
	\begin{split}
	|last(\hat{W}_{t,j}) - p_{j\cdot w-i}| &= |s_{t,j\cdot w} - p_{j\cdot w-i}| \\
	&\leqslant  \theta_{ele}(j\cdot w-i)
	\end{split}
	\end{equation*}	
	Remove absolute value symbol:
	\begin{displaymath}
	-\theta_{ele}(j\cdot w-i) \leqslant last(\hat{W}_{t,j}) - p_{j\cdot w-i} \leqslant \theta_{ele}(j\cdot w-i)
	\end{displaymath}
	Focus on the right inequality, we deduce that:
	\begin{displaymath}
	last(\hat{W}_{t,j}) \leqslant p_{j\cdot w-i} + \theta_{ele}(j\cdot w-i)
	\end{displaymath}
	Consider the definition of $ U $ (Equation~\ref{eq:element_ul}) and $ \hat{P}_j^u $ (Equation~\ref{eq:element_bound}), there are:
	\begin{displaymath}
	last(\hat{W}_{t,j}) \leqslant U_{j\cdot w-i}
	\leqslant
	\max \limits_{0 \leqslant a < w} (U_{j\cdot w - a}) = \hat{P}_j^u
	\end{displaymath}
	Similarly we also prove $last(\hat{W}_{t,j}) \geqslant \hat{P}_j^l$. 
	
	The proof is completed.
\end{proof}

For each block $\hat{W}_{t,j} $, the  time complexities of  computing feature and comparing $\hat{W}_{t,j} \prec \hat{P}_j$ are both $ O(1) $, which makes $ELB_{ele}$ very efficient.
However, it constrains a single element pair by the tolerance of the whole segment, which makes the envelope loose.
Its pruning effectiveness is better when thresholds are relatively small or pattern deviates from normal stream far enough.


\subsection{Subsequence-based ELB Representation} 
\label{sub:lb_avg}
In this section, we introduce the second ELB implementation, subsequence-based ELB, denoted by $ELB_{seq}$. Compared to $ELB_{ele}$,  $ELB_{seq}$ has a tighter bound albeit spends a little more time on computing features of window blocks.
Different from $ELB_{ele}$ which uses the tolerance of the whole segment to constrain one element pair,  in $ELB_{seq}$, we use the same tolerance to constrain a $w$-length subsequence.

Consider two $ w $-length subsequences $ P[i':i] $ and $ W_t[i':i] $ where $ i' = i-w+1 $. We first assume all elements in $ P[i':i] $ (or $ W_t[i':i] $) belongs to only one segment, like $ P_k $ (or $ W_{t,k} $). In other words, the interval $[i':i]$ doesn't overlap with both left and right break regions. Recall that for any segmentation, if it always holds $W_{t,k}\prec P_k$, there is:
\[
ED(W_{t,k},P_k) \leqslant \varepsilon_k \sqrt{  r_{k}-l_{k-1}} = md(k)
\]
So we can obtain:
\begin{equation}
ED(P[i':i],W_t[i':i])^2 =  \sum_{j=i'}^{i}(p_j-s_{t,j})^2 \leqslant md(k)^2
\label{eq:itoi}
\end{equation}

We utilize the mean value of the subsequence to construct the envelope based on following property. Given sequence $X=(x_1,\cdots,x_m)$ and $Y=(y_1,\cdots,y_m)$ , referring to \cite{lian_similarity_2007}, it holds that:
\begin{equation}
\label{eq:convex}
m\left| \mu_x-\mu_y \right|^2 \leqslant \sum_{i=1}^m \left| x_i-y_i \right|^2
\end{equation}	
where $\mu_x$ and $\mu_y$ are the mean values of $X$ and $Y$.

We denote by $\mu_{P[i':i]}$ and $\mu_{W_t[i':i]}$ that the mean value of $P[i':i]$ and $W_t[i':i]$ respectively.  By combining Equation~\ref{eq:itoi} and Equation~\ref{eq:convex}, we have:
\begin{equation}
|\mu_{P[i':i]}-\mu_{W_t[i':i]}| \leqslant md(k)/\sqrt{w}
\end{equation}

We take $\mu_{W_t[i':i]}$ as the feature of subsequence $W_t[i':i]$, and construct the envelope of pattern $P$ as follows:
\begin{equation}
\begin{split}
U_i = \mu_{P[i':i]}+md(k)/\sqrt{w} \\
L_i = \mu_{P[i':i]}-md(k)/\sqrt{w}
\end{split}
\label{eq:onesegment}
\end{equation}

Now we consider the case that the interval $[i':i]$ overlaps with more than one segment. Suppose $[i':i]$ overlaps with the $k_l$-th segment to the $k_r$-th segment.
Due to the additivity of the squared Euclidean distance,  we deduce from Equation~\ref{eq:itoi} that:
\begin{equation}
ED(P[i':i],W_t[i':i])^2
\leqslant \sum_{k=k_l}^{k_r} md(k)^2
\label{eq:dedue-itoi}
\end{equation}
Combined with Equation~\ref{eq:convex}, we have
\begin{equation}
|\mu_{P[i':i]}-\mu_{W_t[i':i]}|
\leqslant \sqrt{\frac{1}{w}\sum_{k=k_l}^{k_r} md(k)^2}
\label{eq:seq_limit}
\end{equation}
We denoted the right term as $\theta_{seq}(i)$ and give the general case of the pattern envelope as follows:
\begin{equation}
\begin{split}
U_i = \mu_{P[i':i]}+\theta_{seq}(i) \\
L_i = \mu_{P[i':i]}-\theta_{seq}(i)
\end{split}
\label{eq:avg_ul}
\end{equation}
Note that Equation~\ref{eq:onesegment} is the special case of Equation~\ref{eq:avg_ul}.

We solve how to process $w$ number of consecutive windows together exactly as $ELB_{ele}$. The construction of upper and lower bounds are very similar to $ELB_{ele}$, while the feature of window block is adopted to the mean value.  We show the basic idea with an example in Figure~\ref{fig:03elb-seq}(a). At $W_{t} (t = i\cdot w + 1, i = 0, 1, 2, \cdots)$, the subsequence $ W_t[7 : 9] $ aligns with $ P[7:9] $. Similarly, in $W_{t+1}$ (or $W_{t+2}$), this subsequence aligns with $ P[6:8] $ (or $ P[5:7] $).
According to Equation~\ref{eq:avg_ul}, we know that if the mean value of $W_t[7 : 9] $ falls out of all upper and lower bounds of $ P[7:9], P[6:8] $ and $ P[5:7] $, these 3 corresponding windows can be pruned together. Based on this observation, we give the formal definition of $ELB_{seq}$ as follows:
\begin{equation}
\begin{cases}
\begin{split}
&\hat{P}_{j}^u = \max \limits_{0 \leqslant i < w} (U_{j\cdot w - i})\\
&\hat{P}_j^l = \min \limits_{0 \leqslant i < w} (L_{j\cdot w - i})\\
&\hat{W}_{t,j}^f = mean(\hat{W}_{t,j}) = \mu_{W_{t}[(j-1)\cdot w +1\, :\, j\cdot w]}
\end{split}
\end{cases}
\label{eq:avg_bound}
\end{equation}
For clarity, we only illustrate the bounds of $ \hat{P}_3 $ in Figure~\ref{fig:03elb-seq}(b). The lower bound $ \hat{P}^l_3 $ is set to the minimum of $ L_7, L_8 $ and $ L_9 $ and covers 3 consecutive windows $ W_{t+1}, W_8 $ and $ W_9 $.
In addition, the bounds of $ \hat{P}_1 $ is meaningless according to the definition of the envelope of $ ELB_{seq} $.
We defer the proof that $ ELB_{seq} $ satisfies the lower bounding property 
in Appendix A 
because it's exactly like the proof of $ ELB_{ele} $.
Similar to $ ELB_{ele} $, $ ELB_{seq} $ can be extended to cases of other $ L_p $-Norm distance measures easily.
\begin{figure}[!htb]
	\centering
	\includegraphics[width=0.47\textwidth]{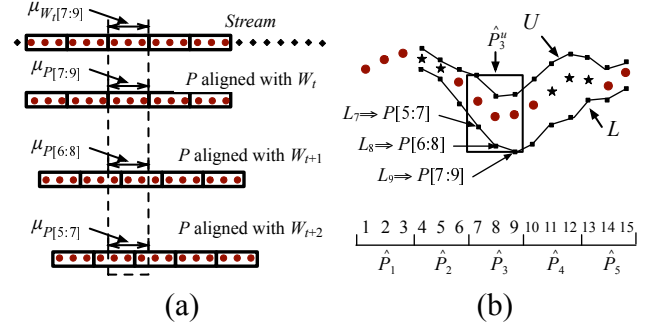}
	\caption{(a) The subsequence $ W_t[7 : 9] $ aligns with $ P[7:9],P[6:8]$ and $P[5:7] $ at $W_{t},W_{t+1}$ and $W_{t+2}$ respectively.
		(b) $ \hat{P}_j^u $ and $ \hat{P}_j^l $ are constructed by $ U $ and $ L $.
	}
	\label{fig:03elb-seq}
\end{figure}

%
%

	\section{Algorithm}
\label{sec:algorithm}
In this section, we first introduce the pruning algorithm based on ELB representation, and then describe the post-processing algorithm.
\subsection{Pruning Algorithm} 
\label{sub:matching}
The pruning algorithm processes the sliding windows based on $w$-length blocks.

\begin{figure}[!htb]
	\centering
	\includegraphics[width=0.325\textwidth]{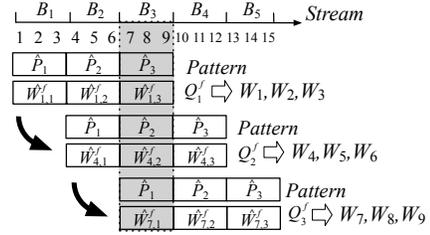}
    \caption{Global blocks will appear in multiple windows. For example, $ B_3 $ corresponds to $ \hat{W}_{1,3},\hat{W}_{4,2} $ and $\hat{W}_{7,1} $  in windows $W_1, W_4$ and $W_7$ successively.
    	$ w $ consecutive windows are processed together.  For example, if any mismatch of blocks in $ Q^f_1 $ occurs, $ W_1, W_2 $ and $ W_3 $ can be pruned together.
    }
	\label{fig:05_1_absolute_block}
\end{figure}
Globally, we split stream $S$ into disjoint blocks $B_1,B_2,\cdots$. The $i$-th block is $ B_i =S[a: a+w-1] $, where $ a = (i-1)\cdot w + 1 $. We call blocks in form of $B_i$ as \textit{global block}s since they are irrelevant to sliding windows. The feature of $B_i$ is denoted as $B^f_i$, which can be computed by either ELB$_{ele}$ or ELB$_{seq}$.
The top of Figure~\ref{fig:05_1_absolute_block} shows an example of global blocks.
Note that each global block will appear in multiple windows. For example, $ B_3 $ corresponds to $ \hat{W}_{1,3}$ in window $W_1$, $\hat{W}_{4,2} $ in window $W_4$ and $ \hat{W}_{7,1} $ in window $W_7$ successively.

We present a simple yet effective pruning algorithm based on the ELB representation, whose pseudo-code is shown in Algorithm~\ref{alg:els_framwork}.
Instead of processing sliding windows one-by-one, we take each new block as an input, and process $w$ number of sliding windows together.
Formally, we maintain an FIFO queue $Q^f$ to contain the features of latest $ N $ blocks.  When $ B^f_i $ is the oldest global block in $Q^f$, we also denote the current $Q^f$ as $Q^f_i$.
Assume the current queue is $ Q^f_{i-1} = \{B_{i-1}, B_{i}, \cdots, B_{i+N-2}\} $.
When a new block $ B_{i+N-1} $ arrives, we update $ Q^f_{i-1} $ to $ Q^f_{i} $ by popping the oldest feature $ B^f_{i-1} $ and appending $ B^f_{i+N-1} $ (Line~\ref{line:alg1-popping}).
Then we compare features in $ Q^f_{i} $ with aligned pattern blocks sequentially (Line~\ref{line:alg1-loop-start} to Line~\ref{line:alg1-loop-end}).
For example, the first feature $ Q^f_{i}[1]  $, corresponding to $ B_{i}^f $, aligns with $ \hat{P}_1 $.
If any mismatch occurs(e.g. $ Q^f_{i}[1] \nprec \hat{P}_1 $), we stop the processing of $ Q^f_{i} $ and the $w$ number of corresponding sliding windows can be pruned together.
Otherwise, we enter the post-processing phase, in which
all sliding windows covered by $ Q^f_{i} $ will be verified one-by-one (Line~\ref{line:alg1-post-processing-start} to Line~\ref{line:alg1-post-processing-end}).

\begin{algorithm}[!htb]
	\caption{Block-Based Matching Algorithm}
	\label{alg:els_framwork}
	\begin{algorithmic}[1] 
		\Require
		$ \hat{P}_1,  \hat{P}_2,  \cdots, \hat{P}_N$: a list of pattern blocks,\quad
		$ Q^f_{i-1} $: a queue containing the last $ N $ features of  window block,\quad
		$B_{i+N-1}$: a new  window block.
		\Ensure Fine-grained matching results
		\State obtain $ Q^f_i $: pop the oldest $ B^f_{i-1} $ and append $ B^f_{i+N-1} $. \label{line:alg1-popping}
		\For{$j = 1 \to N$}\label{line:alg1-loop-start}
			\If { $ Q^f_{i}[j] \nprec \hat{P}_j$ }
				\State \Return
			\EndIf
		\EndFor\label{line:alg1-loop-end}
		\State $ t \gets (i-1)\cdot w+1 $\label{line:alg1-post-processing-start}
		\State conduct post processing for $W_{t}, W_{t+1}, \cdots W_{t+w-1}$. \label{line:alg1-post-processing-end}
	\end{algorithmic}
\end{algorithm}
For instance in Figure~\ref{fig:05_1_absolute_block}, $ Q^f_1 $ contains features of $ B_1, B_2 $ and $ B_3 $ and covers sliding windows $W_1$, $W_2$ and $W_3$ .
If any feature in $ Q^f_1 $ doesn't match its aligned pattern block,
we could prune these three windows together.
Otherwise, we conduct post processing for $W_1$, $W_2$ and $W_3$  and report windows which are fine-grained matching results of $ P $.
Next,
as $ B_4 $ arrives, we update $ Q^f_1 $ to $ Q^f_2 $ by popping $ B^f_1 $ and appending $ B^f_4 $, and then process $ Q^f_2 $  corresponding to windows $ W_4, W_5 $ and $ W_6 $.

Algorithm~\ref{alg:els_framwork} brings us two advantages.
First, since the block size is equal to the sliding step, we only need calculate once for feature of each block , which means the amortized complexity of feature calculation is $ O(1) $.
Second, although we process windows with size of sliding step larger than 1, which sounds like batch solutions in previous works, our algorithm doesn't sacrifice response time.
In contrast to previous works(e.g. \cite{bulut_unified_2005, lian_multiscale_2009}) which have to wait until data of $ w $ sliding windows are collected entirely,
we could prune a sliding window before all of its elements arrive.
As shown in Figure~\ref{fig:05_1_absolute_block}, $ Q^f_1 $, computed by $ S[1:9] $, can be used to prune $ W_3 $($ =S[3:11] $), even though $ s_{10} $ and $ s_{11} $ do not  arrive yet.

\subsection{Post-processing Algorithm}
\label{sec:exact-calc}
In traditional pattern matching works, the post-processing phase is trivial which computes Euclidean distance between the pattern and an unpruned candidate with complexity of $ O(n)$.
However, in our work, the multi-segment setting with variable breakpoints makes it a difficult task. There exists $O(r^b)$ number of possible segmentations where $r$ is the average length of break regions.
Checking all segmentations is prohibitively time-consuming.
We begin with a baseline solution (Section~\ref{subsub:strawman}) whose time complexity is $ O(nr^2) $ in the worst case, and then introduce an adaptive post-processing algorithm with linear complexity of $ O(n) $(Section~\ref{subsub:post}).

\subsubsection{Baseline}
\label{subsub:strawman}
We introduce a baseline solution which determines breakpoints sequentially for each break region instead of exhaustively examining all segmentations for each sliding window.
As the post-processing phase is independent of the timestamp of the sliding window,
for simplicity, we denote the unpruned candidate by $C = \{c_1, c_2, \cdots, c_n\}$ instead of $ W_t $.

We begin by processing the first segment, from which we try to find all potential $ bp_1 $ in $ BR_1 $ which lead to $C_1\prec P_1$. Formally, we denote the potential set of $ bp_1 $ as follows:
\begin{equation}
PS_1=\{j \in [l_1,r_1] | ED_{norm}(P[1:j] : C[1:j]) \leqslant \varepsilon_1\}
\end{equation}
The potential sets of other breakpoints are similar.
If $ PS_1 $ is not empty, we calculate $ PS_2 $ from $ BR_2 $ based on it.
For each $ i $ in $PS_1$ and $ j $ in $ BR_{2} $, if $ C[i+1:j] $ is an $ \varepsilon_2 $-match of $ P[i+1:j] $, we add $ j $ into $ PS_2 $.
After that, if $ PS_2 $ is empty, $ C $ can be abandoned safely. Otherwise, we continue to calculate $ PS_3, PS_4, \cdots, PS_{b-1}$.
If all of $PS_i$ ($1 \leqslant i<b$) is not empty, we examine whether there exists at least one potential breakpoint $ i $ in $ PS_{b-1} $, so that
$C_{b}\prec P_{b}$. If it is the case, $ C $ is reported as a result.

The correctness of the baseline solution is obvious. Its time complexity is calculated as follows.
Processing step for each segment involves up to $ O(r^2) $ distance calculations where each distance calculation is in $ O(n/b) $.
For $ b $ segments, the total time complexity is $ O(nr^2) $ in the worst case.

\subsubsection{Adaptive  Post-processing Algorithm}
\label{subsub:post}
In this section, we present an adaptive post-processing algorithm which allows us to determine only one \textit{optimal breakpoint} instead of a set of potential breakpoints, which greatly decreases the amount of distance calculations.
Our algorithm is able to examine a candidate with $O(n)$ while guaranteeing no false dismissals.

We first give the definition of the optimal breakpoint.
Suppose we have determined $ bp_1, bp_2, \cdots bp_{k-1} $,
and focus on the $ k $-th optimal breakpoint.
To facilitate the description, we define two delta-functions, $\delta(i,\varepsilon)$ and $\Delta(l,r,\varepsilon)$ as follows, where $1\leqslant l\leqslant r\leqslant n$:
\begin{equation}
\begin{split}
\delta(i,\varepsilon)&=|c_i-p_i|^2-\varepsilon^2\\
\Delta(l,r,\varepsilon)&=\sum_{i=l}^{r} \delta(i,\varepsilon)
\end{split}
\label{eq:delta}
\end{equation}
The first property of $\Delta(l,r,\varepsilon)$ is equivalence:
\newtheorem{lemma}{lemma}[section]
\begin{lemma}
	\label{LEMMA:DELTA}
	$\Delta(l,r,\varepsilon) \leqslant 0$ is equivalent to
	\begin{displaymath}
	ED_{norm}(C[l:r], P[l:r]) \leqslant \varepsilon
	\end{displaymath}
\end{lemma}
\begin{proof}
	According to the definition of $ \Delta(l,r,\varepsilon) $, there is:
	\begin{equation*}
	\begin{split}
	\Delta(l,r,\varepsilon) &=\sum_{i=l}^{r}\delta(i,\varepsilon)
	=\sum_{i=l}^{r}(|c_i-p_i|^2-\varepsilon^2) \\
	&=\sum_{i=l}^{r}|c_i-p_i|^2-(r-l)\varepsilon^2 \leqslant 0\\
	\end{split}
	\end{equation*}
	Transpose $ (r-l)\varepsilon^2 $:
	\begin{displaymath}
	\sum_{i=l}^{r}|c_i-p_i|^2 \leqslant (r-l)\varepsilon^2
	\end{displaymath}
	Transpose $ (r-l) $ and square both sides:
	\begin{displaymath}
	\sqrt{\frac{1}{r-l}\sum_{i=l}^{r}|c_i-p_i|^2}
	= ED_{norm}(C[l:r],P[l:r])
	\leqslant \varepsilon
	\end{displaymath}
	
	The proof is completed.
\end{proof}
The second property of $\Delta(l,r,\varepsilon)$ is additivity:
\begin{displaymath}
\Delta(l,r,\varepsilon)=\Delta(l,r-1,\varepsilon)+\delta(r,\varepsilon)=
\delta(l,\varepsilon) + \Delta(l+1,r,\varepsilon)
\end{displaymath}

Now we can define the $k$-th optimal breakpoint as follows:
\begin{definition}(\textbf{Optimal breakpoint})
	Given $ PS_k $, the $k$-th breakpoint is named as an optimal breakpoint, denoted by $bp^{opt}_k$, if it holds that:
	\begin{displaymath}
	bp^{opt}_k=\mathop{\arg\min}_{j\in PS_{k}} \Delta(j+1,r_k+1,\varepsilon_{k+1})
	\end{displaymath}
	\label{def:opt-bp}
\end{definition}

We prove correctness of our algorithm with Lemma~\ref{LEMMA:POSTPROCESS}.
\begin{lemma}
	\label{LEMMA:POSTPROCESS}
	Assume $bp_{k+1}$ is any breakpoint in $[l_{k+1},r_{k+1}]$. If $\Delta(bp^{opt}_{k}+1,bp_{k+1},\varepsilon_{k+1})>0$, then for any other breakpoint $j\in PS_{k}$, it holds that $\Delta(j+1,bp_{k+1},\varepsilon_{k+1})>0$.
\end{lemma}
\begin{proof}
	Based on the definition of $ \Delta(l,r,\varepsilon) $, for $ bp_k^{opt} $:
	\begin{equation*}
	\begin{split}
	\Delta(bp^{opt}_{k}+1,bp_{k+1},\varepsilon_{k+1})
	&=\Delta(bp^{opt}_{k}+1,r_{k}+1,\varepsilon_{k+1})\\
	&+\Delta(r_{k}+2,bp_{k+1},\varepsilon_{k+1})
	\end{split}
	\end{equation*}
	For $ j $:
	\begin{equation*}
	\begin{split}
	\Delta(j+1,bp_{k+1},\varepsilon_{k+1})
	&=\Delta(j+1,r_{k}+1,\varepsilon_{k+1})\\
	&+\Delta(r_{k}+2,bp_{k+1},\varepsilon_{k+1})
	\end{split}
	\end{equation*}
	According to Definition~\ref{def:opt-bp}, there is:
	\begin{displaymath}
	\Delta(j+1,r_{k}+1,\varepsilon_{k+1})
	\geqslant
	\Delta(bp^{opt}_{k}+1,r_{k}+1,\varepsilon_{k+1})
	\end{displaymath}
	Therefore:
	\begin{displaymath}
	\Delta(j+1,bp_{k+1},\varepsilon_{k+1})
	\geqslant
	\Delta(bp^{opt}_{k}+1,bp_{k+1},\varepsilon_{k+1})
	>0
	\end{displaymath}	
	
	The proof is completed.
\end{proof}

Lemma~\ref{LEMMA:POSTPROCESS} demonstrates that we just need to maintain an optimal breakpoint instead of a set of potential values.
Now we introduce how to find $PS_k$ and $ bp^{opt}_k $ in linear time. Assume we have determined $bp^{opt}_{k-1}$\footnote{In the case of $k=1$,we just set $bp^{opt}_{0}=0$.}, thus the starting point of the $ k $-th segment is $L=bp^{opt}_{k-1}+1$. To compute $PS_{k}$, we first initialize $\Delta(L,L,\varepsilon_k) = \delta(L,\varepsilon_k)$ and then extend it rightward to obtain $\Delta(L,L+1,\varepsilon_k)$, $\Delta(L,L+2,\varepsilon_k), \cdots$ until $\Delta(L,r_{k},\varepsilon_k)$. For each $i\in [l_{k},r_{k}]$, if there is $\Delta(L,i,\varepsilon_k)\leqslant 0$, we add $i$ into $PS_k$.

\begin{figure}[!htb]
	\centering
	\includegraphics[width=0.475\textwidth]{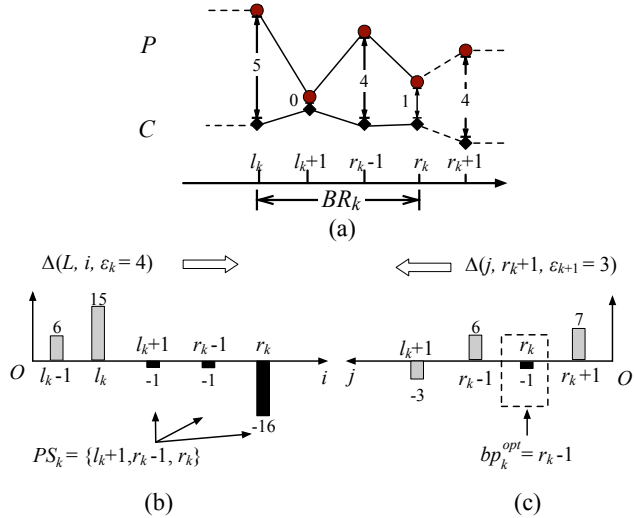}
	\caption{
		(a) The distances between 4 aligned element pairs within $ BR_k $ are $ 5, 0, 4$ and $ 1$ respectively.
		(b) Compute $\Delta $-value from left to right and then $\{l_k+1, r_k-1, r_k\} $ are added into $ PS_k $.
		(c) Compute $\Delta $-value from right to left. $ r_k - 1 $ is determined as  $bp_{k}^{opt}$ since it has the minimal $\Delta$-value among all breakpoints in $ PS_k $.
	}
	\label{fig:06-post-processing}
\end{figure}
For example,
Figure~\ref{fig:06-post-processing}(a) illustrates distances of element pairs within $ BR_k $
and Figure~\ref{fig:06-post-processing}(b) depicts the process of finding $ PS_k $.
Given $ \varepsilon_k = 4 $, we accumulate the delta-function from $\Delta(L,l_{k},\varepsilon_k)  = 15$ to $\Delta(L,r_{k},\varepsilon_k) = -16$ . Finally, we obtain $ PS_k = \{l_k+1, r_k-1, r_k\} $.

If $PS_k$ is not empty, we determine $bp_{k}^{opt}$ from it.
We first initialize $\Delta(r_k+1,r_k+1,\varepsilon_{k+1})$ as $\delta(r_k+1,\varepsilon_{k+1})$, which corresponds to the rightmost potential breakpoint $ r_k $. Then we extend it leftward and obtain $\Delta(r_k,r_k+1, \varepsilon_{k+1})$, $\Delta(r_k-1,r_k+1, \varepsilon_{k+1}), \cdots$, until we reach $\Delta(l_k+1,r_k+1,\varepsilon_{k+1})$, corresponding to the leftmost potential breakpoint $ l_k $. In this process, we keep the breakpoint which belongs to $ PS_k $ and has the minimal $\Delta$-value until the process terminates, and the final one is $bp^{opt}_k$.

Figure~\ref{fig:06-post-processing}(c) illustrates the process of determining $bp^{opt}_k$.
We begin with  $\Delta(r_k+1,r_k+1,\varepsilon_{k+1}) = 7$  and extend it leftward until  $\Delta(l_k+1,r_k+1,\varepsilon_{k+1}) = -3$.
$ r_k - 1 $ is determined as $ bp^{opt}_k$ because it has the minimal $\Delta$-value, $ \Delta(r_k,r_k+1, \varepsilon_{k+1}) = -1$, among all breakpoints in $ PS_k $.

A candidate $C$ is reported as a result only if we successfully find all of $b-1$ number of optimal breakpoints, and in the $b$-th segment, $C_b\prec P_b$ where $C_b=(c_{bp^{opt}_{b-1}+1},\cdots, c_n)$ and $P_b=(p_{bp^{opt}_{b-1}+1},\cdots, p_n)$. Otherwise, $C$ cannot match $P$ for any segmentations.
Even if in the worst case, our post-processing algorithm determines a breakpoint by scanning its break region twice, where the time complexity is $ O(n) $.


	\section{Block-Skipping Pruning}\label{sec:smp}
In this section, we propose an optimization, named Block-Skipping Pruning(BSP), which utilizes \textit{block-skipping set} and \textit{lookup table} to speed up the pruning phase.

\noindent
\textbf{Block-Skipping Set}
As illustrated in Algorithm~\ref{alg:els_framwork}, each window block is appended into the queue $ Q^f $ from one side,
aligns with each pattern block of $P$ successively and is finally popped out from the other side.
Since pattern blocks are static,
as soon as a new window block $ B_{i} $ is appended into $ Q^f $,  we can determine whether it matches pattern blocks $ \hat{P}_N, \hat{P}_{N-1}, \cdots, \hat{P}_1$. Once $ B_{i} $ doesn't match $ \hat{P}_j $, we can safely prune corresponding $ Q^f $ in which $ B_{i} $ aligns with $ \hat{P}_j $, even though the entire $Q^f$ doesn't arrive yet. In other words, the new $ B_{i} $ can prune multiple $Q^f$, as well as the corresponding sliding windows.

We use block-skipping set to represent the indexes of $Q^f$ which can be pruned by block $B_i$. Formally,
\begin{equation}
skip\_set(B_i) = \{i-j+1| B_{i} \nprec \hat{P}_j, j\in [1, N]\}
\label{eq:bs_set}
\end{equation}
Consider the example in Figure~\ref{fig:skip_set}, the block $B_3$ with feature of 4 matches $ \hat{P}_3 $ but doesn't match $ \hat{P}_1 $ and $ \hat{P}_2 $. Therefore $ skip\_set(B_3) = \{2, 3\} $, which indicates that $ Q^f_2 $ and $ Q^f_3 $ can be skipped directly.
\begin{figure}[!htb]
	\centering
	\includegraphics[width=0.375\textwidth]{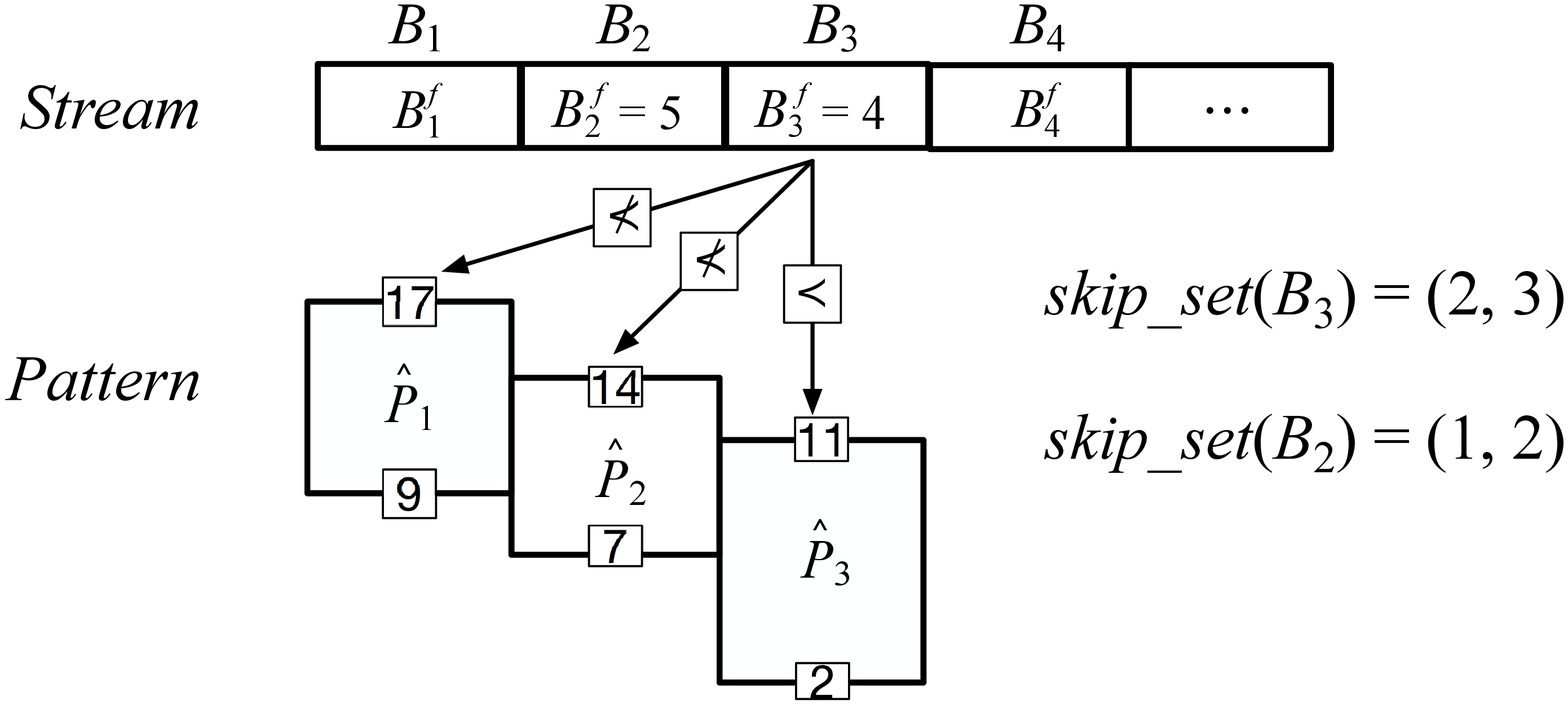}
	\caption{
		The definition of block-skipping set. $ B^f_3  = 4 $ doesn't match $ \hat{P}_1 $ and $ \hat{P}_2 $, therefore $ skip\_set(B_3) = \{2, 3\} $, which means that $ Q^f_2 $ and $ Q^f_3 $ can be skipped directly.
	}
	\label{fig:skip_set}
\end{figure}

\noindent
\textbf{Lookup Table}
According to Equation~\ref{eq:bs_set}, we may have to exhaustively compare $ B_i $ with all pattern blocks to obtain $ skip\_set(B_i) $.
Nevertheless, since all of upper and lower bounds of pattern blocks can be calculated in advance, we construct a \textit{lookup table} to speedup the computation of block-skipping sets.

Specifically, we divide the feature domain, $ (-\infty, \infty) $, into $A$ number of disjoint regions, $\{VR_1,\cdots,VR_A\}$, based on upper and lower bounds of pattern blocks (i.e., $ \hat{P}^u_j $ and $ \hat{P}^l_j $). We sort these bounds in ascending order and denoted them by $\{\beta_1, \beta_2,\cdots,\beta_{A-1}\}$.\footnote{To ease the description, we add two dummy points, $\beta_0=-\infty$ and $ \beta_A=\infty$.}
For a block feature, $ f $, we use $region(f)$ to denote the region $f$ falls within, that is,
$region(f)=VR_i$, if $\beta_{i-1}<f<\beta_{i}$. As a special case, if $f$ exactly equals to one boundary value, say $\beta_i$, $ region(\beta_i) $ will be
the union of $ VR_i $ and $ VR_{i+1}$ ($1\leq i\leq A-1$).
As shown in Figure~\ref{fig:card_red_ipd}(a), the feature domain is divided into 7 value regions along with bounds of pattern blocks $\{2, 7, 9, 11, 14, 17\}$. For $ B_2^f = 5$, there is $ region(B^f_2)=VR_2 $.

Lookup table records the relationship between value regions and pattern blocks. Each entry is a key-value pair, in which the key is a value region, and the value is the set of pattern block indexes. Formally,
\begin{equation}
table(VR_i) = \{j|VR_i \nprec \hat{P}_j, j \in [1, N]\}
\end{equation}
Figure~\ref{fig:card_red_ipd}(b) shows the lookup table of Figure~\ref{fig:card_red_ipd}(a). Since $ VR_2 $ overlaps with $ \hat{P}_3 $, $ table(VR_2) =\{1,2\}$ . Similarly, $ table(VR_7) $ is $\{1, 2, 3\}$ since $ VR_7 $ does not overlap with any pattern block.
Based on the lookup table, we compute $ skip\_set(B_i) $ as follows:
\begin{equation}
skip\_set(B_i) = \{i-j+1| j\in table(region(B^f_i)) \}
\label{eq:bs_set_r}
\end{equation}
\begin{figure}[!htb]
	\centering
	\includegraphics[width=0.475\textwidth]{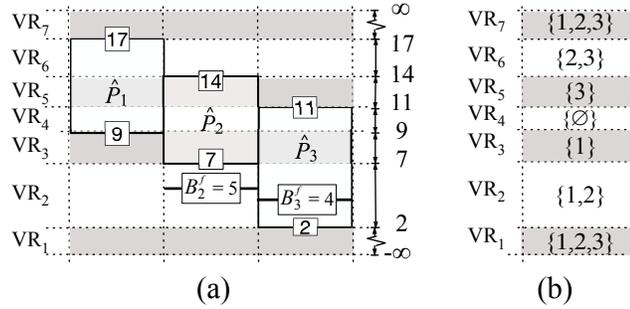}
	\caption{
		(a) The feature domain is divided into 7 value regions along with $\{2, 7, 9, 11, 14, 17\}$. $ B^f_2 = 5$ and $ B^f_3 = 4$ are mapped to $ VR_2 $.
		(b) A lookup table of pattern blocks in (a). $ table(VR_2) =\{1,2\}$ since $ VR_2 $ overlaps with $ \hat{P}^f_3 $ .
	}
	\label{fig:card_red_ipd}
\end{figure}

\begin{algorithm}[!htb]
	\caption{Block-Skipping Pruning Optimization}
	\label{alg:bsp_algorithm}
	\begin{algorithmic}[1] 
		\Require
		$ Q^f_{i} $: a queue containing the last $ N $ features of  window blocks, \quad
		$ skip\_union $: a union containing indexes of $ Q^f $ which can be skipped in the future.
		\Ensure Fine-grained matching results
		\State update $ Q^f_i $ to $ Q^f_{i'} $: $ i'$ is the nearest index where $ i' \geqslant i $ and $ i' \notin skip\_union $.\label{line:bsp-update-qf}
		\For{$j = N \to 1$}\label{line:get-bsp-loop-start}
		\State $ \mathcal{D} \gets skip\_set(Q^f_{i'}[j]) $
		\State merge $ \mathcal{D} $ into $ skip\_union $.
		\If { $ i' \in \mathcal{D}$ }
		\State \Return
		\EndIf
		\EndFor\label{line:get-bsp-loop-end}
		\State $ t \gets (i'-1)\cdot w+1 $
		\State conduct post processing for $W_{t}, W_{t+1}, \cdots W_{t+w-1}$. \label{line:bsp-post-processing-end}
	\end{algorithmic}
\end{algorithm}

\noindent
\textbf{BSP Optimization}

Based on block-skipping set, we propose BSP optimization algorithm to accelerate the pruning phase, instead of re-comparing features of window block from scratch. The extended algorithm is shown in Algorithm~\ref{alg:bsp_algorithm}.
We maintain a global skipping set, $ skip\_union $, to record indexes of $ Q^f $ which can be skipped in the future. The global skipping set is the union of those of all recent window blocks.
Assume we arrive $ Q^f_i $, we first compute $skip\_set(B_{i+N-1})$ and merge it into $skip\_union$. Then we check whether $i\in skip\_union$, if it is the case, we stop the processing of $ Q^f_i $, and come to the nearest $Q^f_{i'}$ which satisfies $i'\notin skip\_union$. Otherwise, we compute $skip\_set(B_{i+N-2})$. This process terminates either $i\in skip\_union$ or the oldest block $B_{i}$ has been checked. If in the latter and we still cannot prune $ Q^f_i $, we enter the post-processing phase to check whether the sliding windows corresponding to $Q^f_i$ will match $P$.

For example in Figure~\ref{fig:skip_set}, $ skip\_union$ is initially empty.
Given $ Q^f_1 $, we compute $ skip\_set(B_{3}) = \{2, 3\} $ and merge it into $ skip\_union$.
As $ skip\_set(B_{3})$ cannot skip $ Q^f_1 $, we compute $ skip\_set(B_{2}) = \{1, 2\} $ continuously.
$ skip\_set(B_{2}) $ contains $ 1 $ so we skip $ Q^f_1 $.
Since current $ skip\_union = \{1, 2, 3\}$, we skip $ Q^f_1, Q^f_2, Q^f_3 $ together and start to process $ Q^f_4 $.

	\section{Experimental Evaluation}
\label{sec:experiment}
In this section we first describe datasets and experimental settings in Section~\ref{sub:experimental_Settings} and then
present the results of performance evaluation comparing Sequential Scanning(SS), MSM and our two approaches based on $ ELB_{ele} $(ELB-ELE) and $ ELB_{seq} $ (ELB-SEQ) respectively. 
Our goal is to:
\begin{itemize}
	\item Demonstrate the efficiency of our approach on different break region sizes and different distance thresholds.
	\item Demonstrate the robustness of our approach on different pattern occurrence probabilities. 
	\item Investigate the impact of block size and BSP policy.
\end{itemize}

\subsection{Experimental Setup}
\label{sub:experimental_Settings}
The experiments are conducted on both synthetic and real-world
 datasets.

\noindent \textbf{Datasets}.
Real-world datasets are collected from a wind turbine manufacturer, where each wind turbine has hundreds of sensors generating streaming time series with sampling rate from 20 ms to 7s.
Our experimental datasets are from 3 turbines.
In each turbine we collect data of 5 sensors including wind speed, wind deviation, wind direction, generator speed and converter power. 
We replay the data as streams with total lengths of $10^8$.
For each stream, a pattern is given together with the break regions and varied multiple thresholds by domain experts.

Synthetic datasets are based on UCR Archive \cite{keogh_welcome_nodate}.
UCR Archive is a popular time series repository, which includes a set of datasets widely used in time series mining researches ~\cite{begum_rare_2014, keogh_exact_2002, lian_similarity_2007}. We generate the synthetic datasets based on UCR Archive. To simulate patterns with diverse lengths, we select four datasets, Strawberry, Meat, NonInvasiveFatalECG\_Thorax1 (ECG for short) and MALLAT whose time series lengths are 235, 448, 750 and 1024, and numbers of pattern segments are 5, 6, 8 and 7 respectively.
For each selected UCR dataset, we always choose the first time series of class 1 as fine-grained pattern.
To obtain patterns with fuzzy regions, 
we first divide a pattern with a list of fixed breakpoints according to its shape and trend, and then extend these breakpoints symmetrically to both sides with a same radius.
For example, the pattern of Strawberry dataset with length of $ 235 $ is divided into five segments by four breakpoints $ \{43, 72, 140, 159\} $.
When the radius is 2, there is $BR_1 = [41, 45]$ and the others are similar.

In terms of threshold, 
we define \textit{threshold\_ratio} as the ratio of its normalized Euclidean threshold to the value range of this segment. 
Given a unified \textit{threshold\_ratio}, we can calculate thresholds of all pattern segments.
In practice, we observe that \textit{threshold\_ratio} being larger than $ 30\% $ indicates that the average deviation from a stream element to its aligned pattern element is more than $ 30\% $ of its value range.
In this case, the candidate may be quite different from given pattern where similarity matching becomes meaningless. 
Therefore, we vary \textit{threshold\_ratio} from 5\% to 30\% in the experiments (Section~\ref{sub:performance_of_different_thres_factor}).

As for streaming data, we first generate a random walk time series $S$ with length 10,000,000 for each selected UCR dataset.
Element $ s_i $ of $ S $ is $s_i = R + \sum_{j=1}^{i} (\mu_j - 0.5)$, where $\mu_j$ is a uniform random number in $ [0, 1]$.
As value ranges of the four patterns are about -3 to 3, we set $R$ as the mean value $ 0 $.
Then we randomly embed some time series of class 1 of each UCR dataset into corresponding steaming data with certain occurrence probabilities \cite{begum_rare_2014}.
These embedded sequences are regarded as possible correct candidates since they belong to the same class as the corresponding pattern.

\noindent \textbf{Algorithm}.
We compare our approaches to SS and MSM \cite{lian_multiscale_2009}.
SS matches each window sequentially by the baseline solution, mentioned in Section~\ref{subsub:strawman}.
MSM builds a hierarchical grid index for the pattern and computes a multi-scaled representation for each window.
We adopt its batch version where the batch size is equal to ELB block size for fair comparison, 
and perform three schemes of MSM to choose the best one: stop the pruning phase at the first level of its grid index(MSM-1), the second level (MSM-2), or never early stop the pruning phase(MSM-MAX).
Although MSM cannot be used for fine-grained pattern directly,  
we can fix it by transforming a fine-grained pattern into a series of fixed-length patterns with a larger threshold. 
For a fine-grained pattern $ P $, we traverse it to $ b $ number of segments $ P'_1, P'_2, \cdots, P'_{b} $, where $ P'_k = P[r_{k-1}+1:l_k] $ and $ \varepsilon'_k =  \varepsilon_k \sqrt{\frac{r_{k}-l_{k-1}}{l_{k}-r_{k-1}}} $. We prove that 
this transformation guarantees no false dismissals 
in Appendix B.
In the pruning phase, MSM adopts the baseline solution same as SS.
%

\noindent \textbf{Default Parameter Settings}.
There are three parameters for synthetic datasets: break region size, distance threshold and pattern occurrence probability, and two parameters for our algorithm: block size and BSP policy.
The default size of break region is set to $ 10\% $ of the shortest length among all segments in a pattern.
The default value of \textit{threshold\_ratio} is set to $ 20\% $ and $ 10^{-4} $ respectively.
In terms of algorithm parameters, we set the default value of block size to $ 5\% $ of the pattern length and enable BSP policy.
The impact of all above parameters will be investigated in following sections.

\noindent \textbf{Performance Measurement}.
We measure the average time of all algorithms for processing each sliding window.
Streams and patterns are loaded into memory in advance where data loading time is excluded.
To avoid the inaccuracy due to cold start and random noise, we run all algorithms over 10,000 ms and average them by their cycle numbers.
All experiments are run on 4.00 GHz Intel(R) Core(TM) i7-4790K CPU, with 8GB physical memory.

\subsection{Performance Analysis}
\label{sub:performance of_different_distances}
In this set of experiments, we demonstrate efficiency and robustness of our algorithm. We vary the size of break region (Section~\ref{sub:performance_of_different_boundary_range}), distance threshold (Section~\ref{sub:performance_of_different_thres_factor}) and pattern occurrence probability (Section~\ref{sub:comparision_of_matching_ratio}). 
We also evaluate the impact of block size (Section~\ref{sub:performance_block_size}) and BSP policy comparing to Algorithm~\ref{alg:els_framwork}(Section~\ref{sub:bsp_investigation}).

\subsubsection{Impact of Break Region}
\label{sub:performance_of_different_boundary_range}

In this section, we investigate the impact of the size of break region on performance over all synthetic datasets.

The results are shown in Figure~\ref{fig:7_boundary}.
The x-axis represents the ratio of the size of break region to the length of the shortest pattern segment, which varies from $ 10\% $ to $ 70\% $.
Our two algorithms, ELB-ELE and ELB-SEQ, outperform MSM and SS on all datasets. 
When the ratio of break region size varies from $ 10\% $ to $ 70\% $, the time cost of SS increases faster than MSM and our algorithms. 
In other words, when the size of break region increases, MSM and our algorithms achieve larger speedup comparing to SS.
\begin{figure}[!htb]
	\centering
\begin{subfigure}[b]{0.225\textwidth}
	\centering
	\includegraphics[width=1\textwidth]{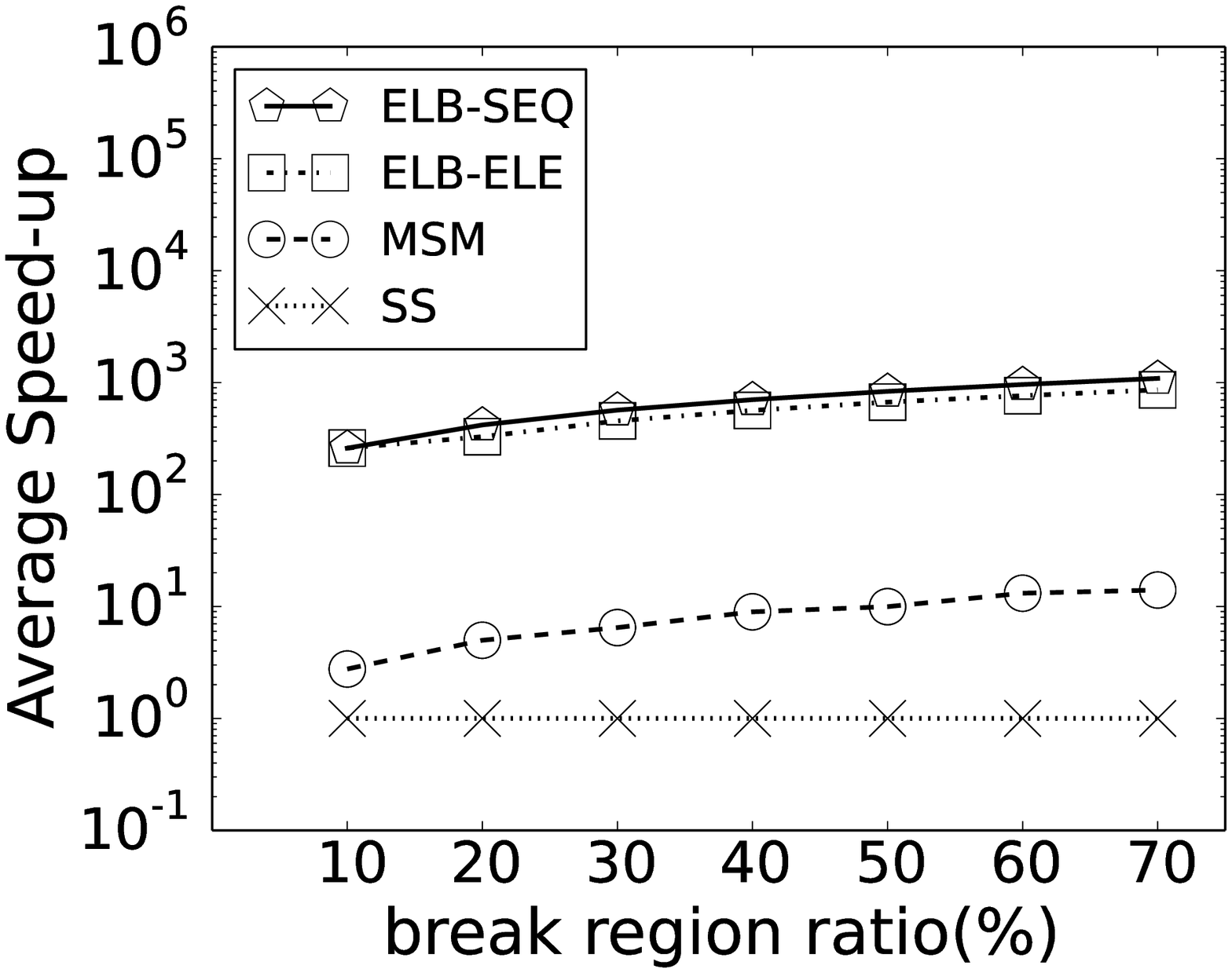}
	\caption[Network2]%
	{UCR\_Strawberry}
\end{subfigure}
\hfill
\begin{subfigure}[b]{0.225\textwidth}  
	\centering 
	\includegraphics[width=1\textwidth]{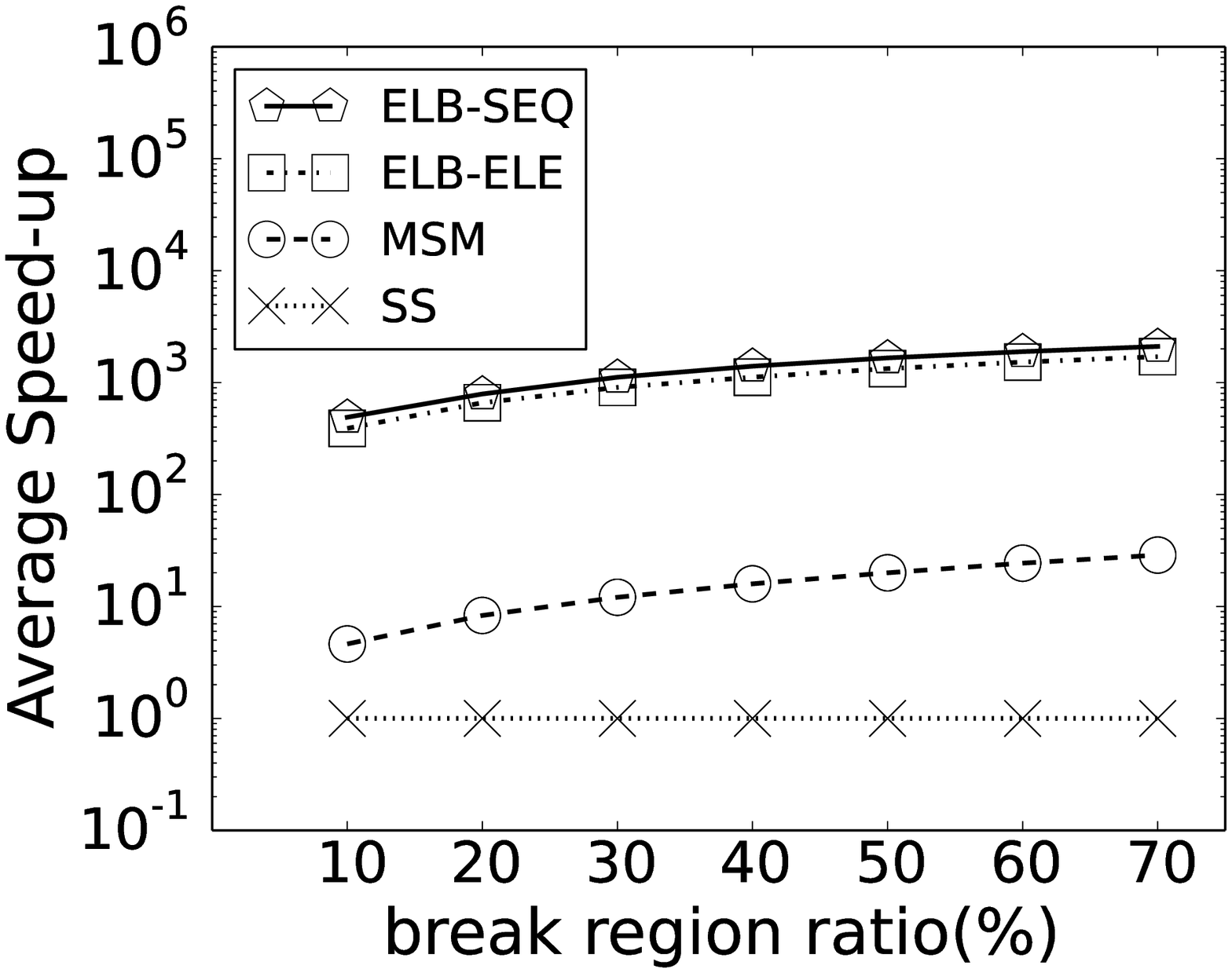}
	\caption[]%
	{UCR\_Meat}
\end{subfigure}
\begin{subfigure}[b]{0.225\textwidth}   
	\centering 
	\includegraphics[width=1\textwidth]{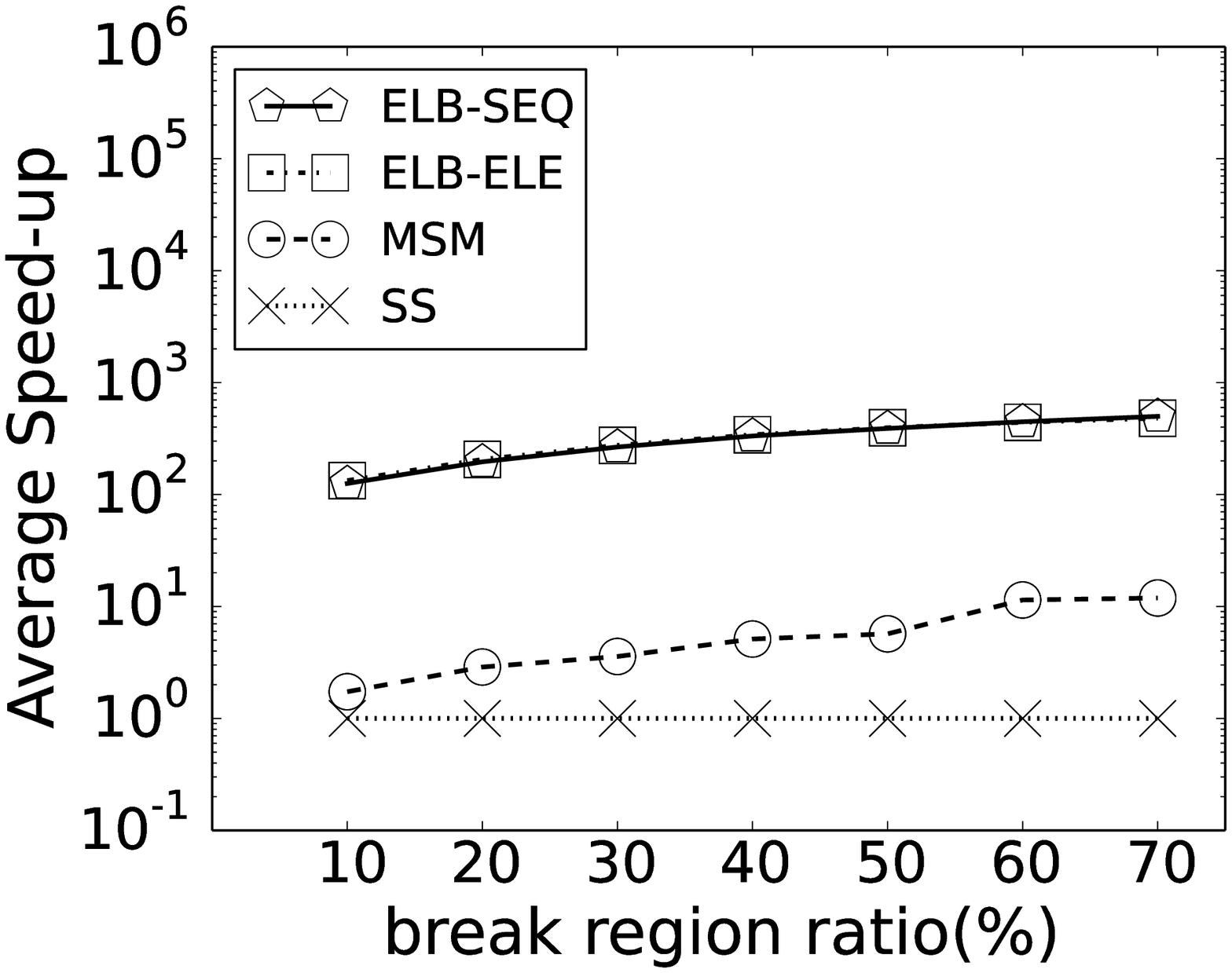}
	\caption[]%
	{UCR\_ECG}    
\end{subfigure}
\hfill
\begin{subfigure}[b]{0.225\textwidth}   
	\centering 
	\includegraphics[width=1\textwidth]{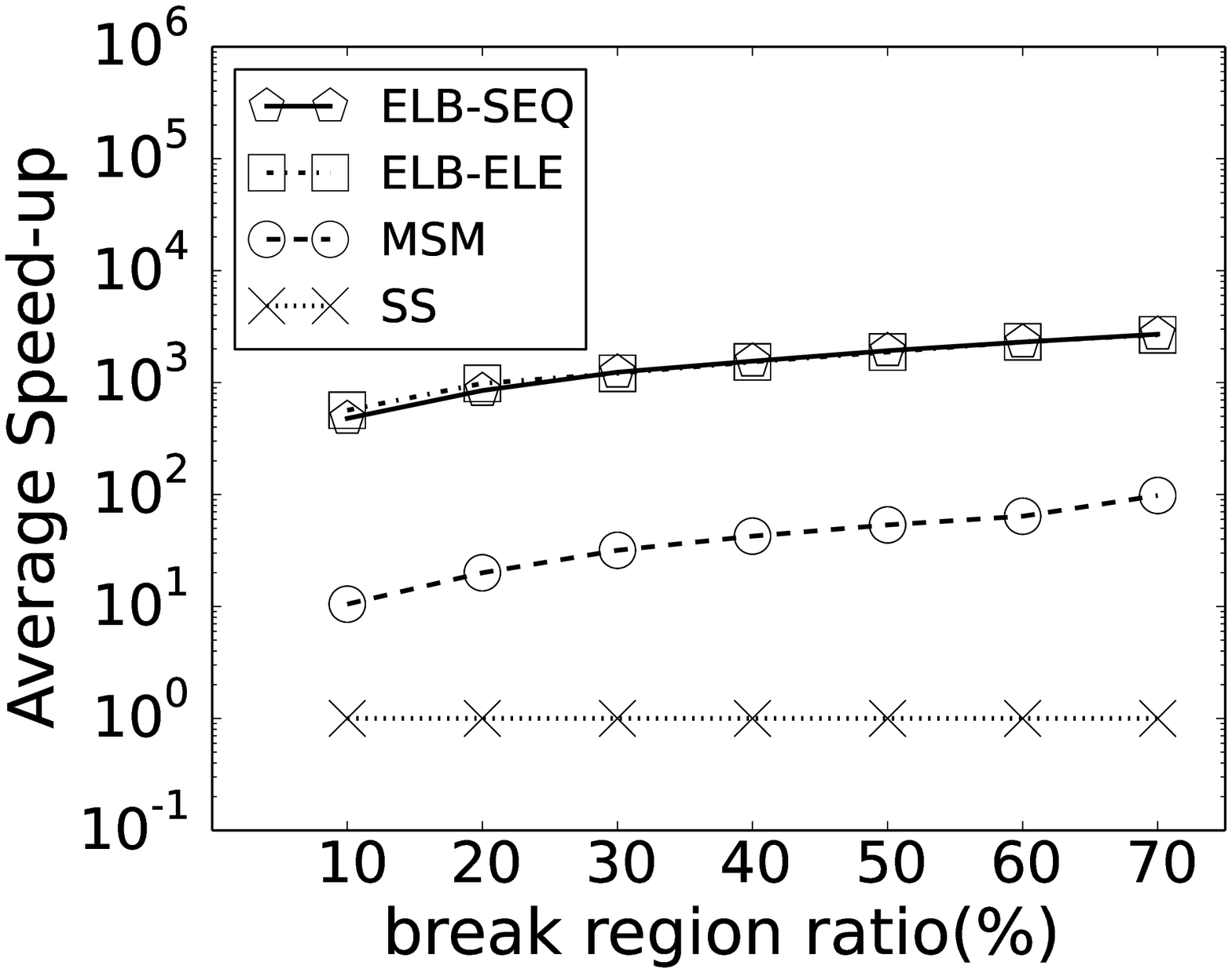}
	\caption[]%
	{UCR\_MALLAT}
\end{subfigure}
	\caption{Speedup vs. break region ratio. 
}
	\label{fig:7_boundary}
\end{figure}

Comparing to SS, our algorithms could prune numerous windows in the pruning phase, while SS has to perform exact matching for each sliding window, resulting in high time cost.
Comparing to SS, the pruning phase of our algorithms is much more efficient than MSM. Although our approaches have more false candidates, benefiting from our adaptive post-processing algorithm, our time cost
of the post-processing phase is similar to MSM.
As a result, our algorithms has an advantage of more than one order of magnitude over MSM.

\subsubsection{Impact of Distance Threshold}
\label{sub:performance_of_different_thres_factor}
In this experiment, we compare the performance of ELB-ELE, ELB-SEQ, SS and MSM when varying distance threshold.
For synthetic datasets, as mentioned in Section~\ref{sub:experimental_Settings}, we vary \textit{threshold\_ratio} from 5\% to 30\%.
For real-world datasets, Domain experts have pre-defined normalized Euclidean thresholds for segments of all patterns. To investigate the impact of threshold, we vary the ratio of the testing threshold to the pre-defined value from 80\% to 120\%. 

The result on synthetic datasets is shown in Figure~\ref{fig:4_threshold}. 
The performances of our two algorithms are very similar in synthetic datasets.
Both ELB-ELE and ELB-SEQ outperforms MSM and SS by orders of magnitude.
As the threshold gets larger, the speedups of ELB-ELE and ELB-SEQ decrease slightly.
Nevertheless, our algorithms keep their advantage over other approaches even though \textit{threshold\_ratio} increases to 30\%.

\begin{figure}[!htb]
	\centering
	\begin{subfigure}[b]{0.225\textwidth}
		\centering
		\includegraphics[width=1\textwidth]{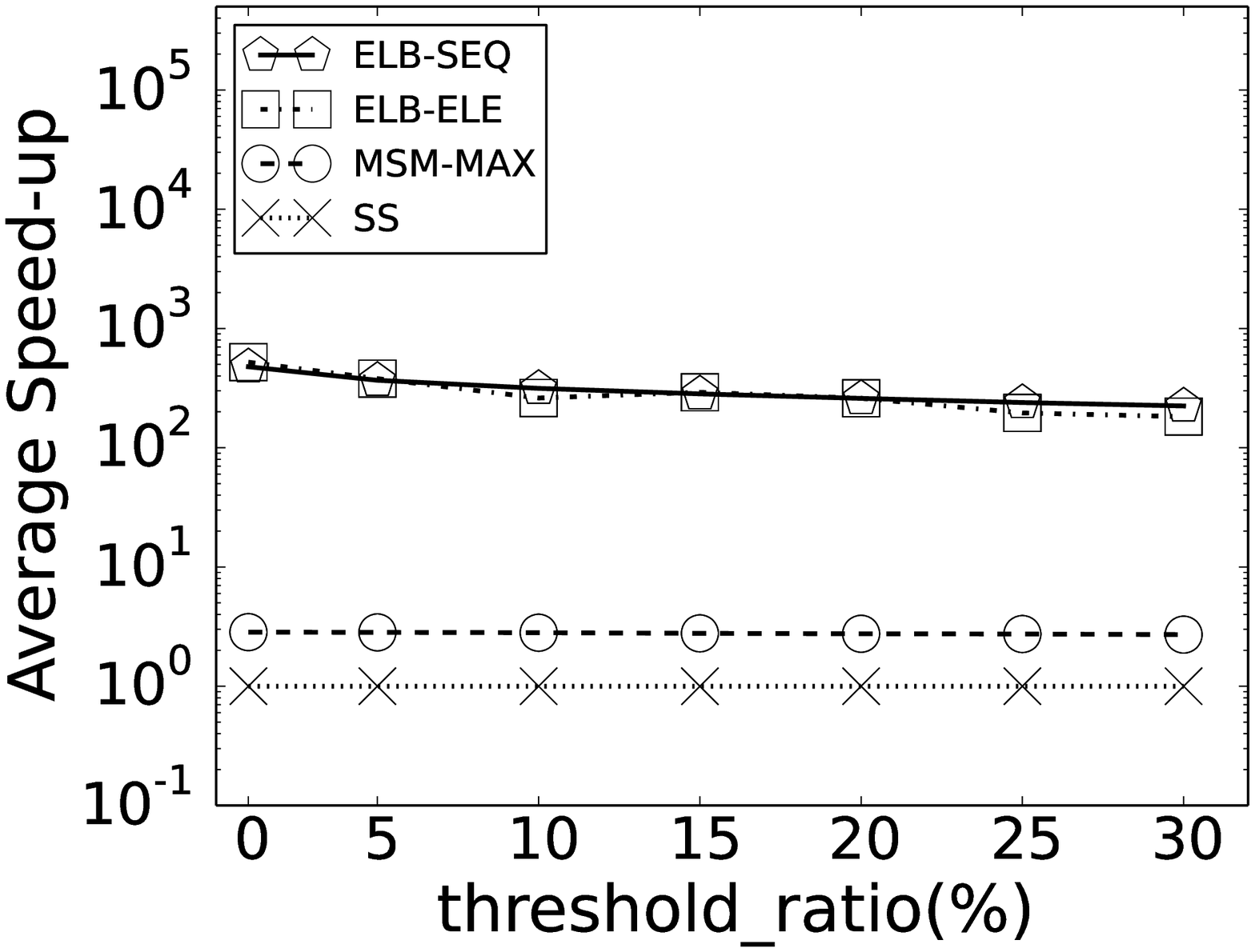}
		\caption[Network2]%
		{UCR\_Strawberry}    
	\end{subfigure}
	\hfill
	\begin{subfigure}[b]{0.225\textwidth}  
		\centering 
		\includegraphics[width=1\textwidth]{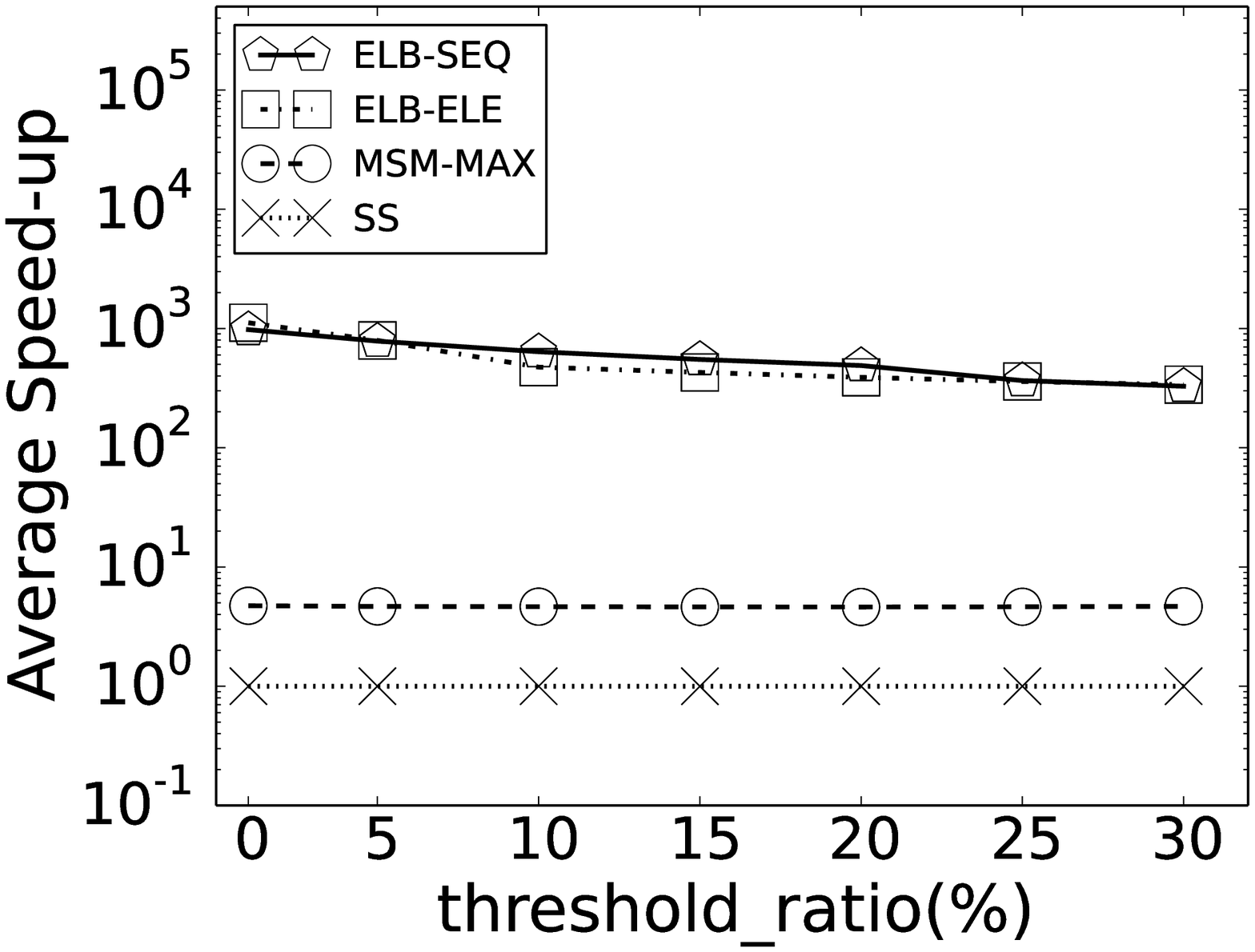}
		\caption[]%
		{UCR\_Meat}    
	\end{subfigure}
	\begin{subfigure}[b]{0.225\textwidth}   
		\centering 
		\includegraphics[width=1\textwidth]{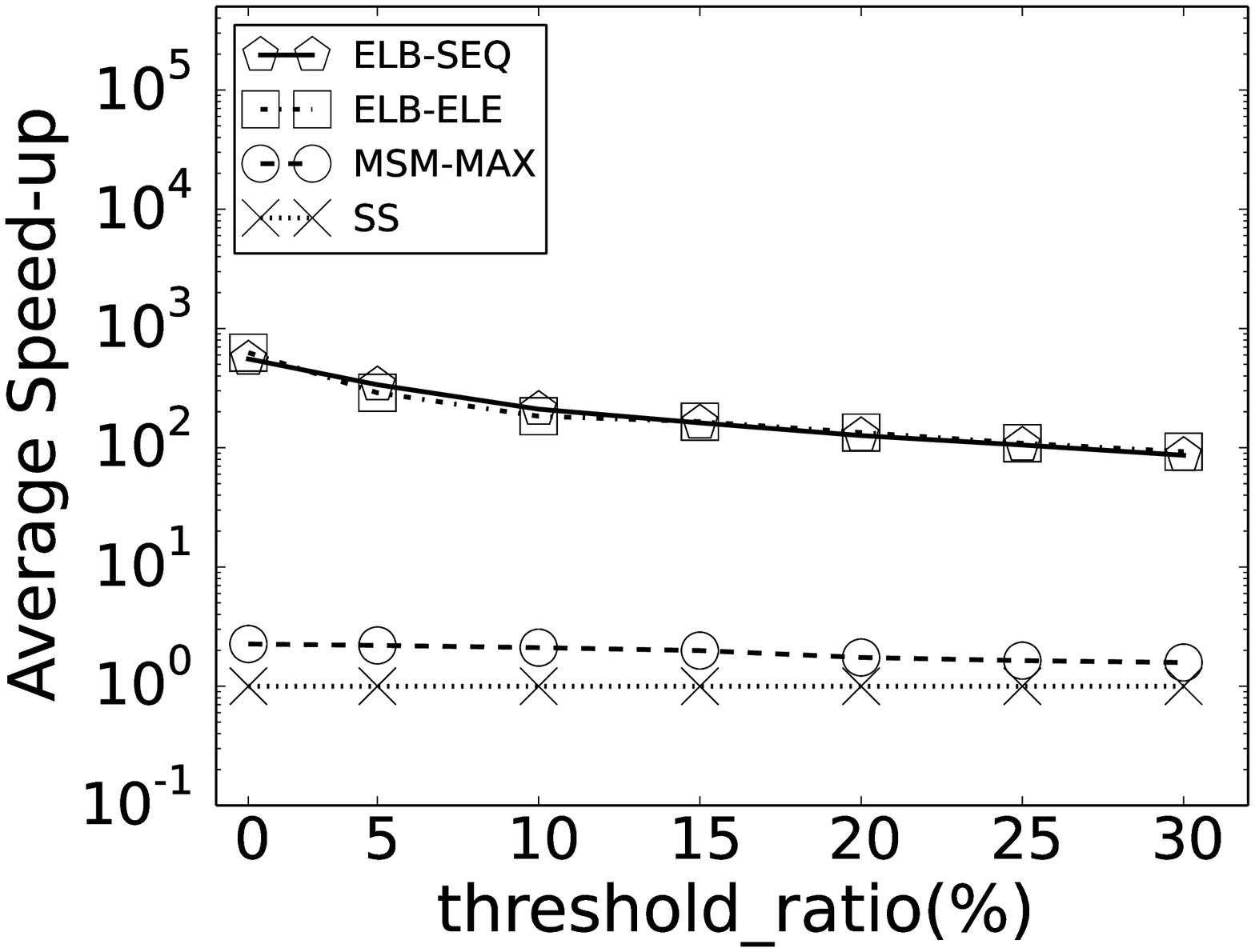}
		\caption[]%
		{ UCR\_ECG}
	\end{subfigure}
	\hfill
	\begin{subfigure}[b]{0.225\textwidth}   
		\centering 
		\includegraphics[width=1\textwidth]{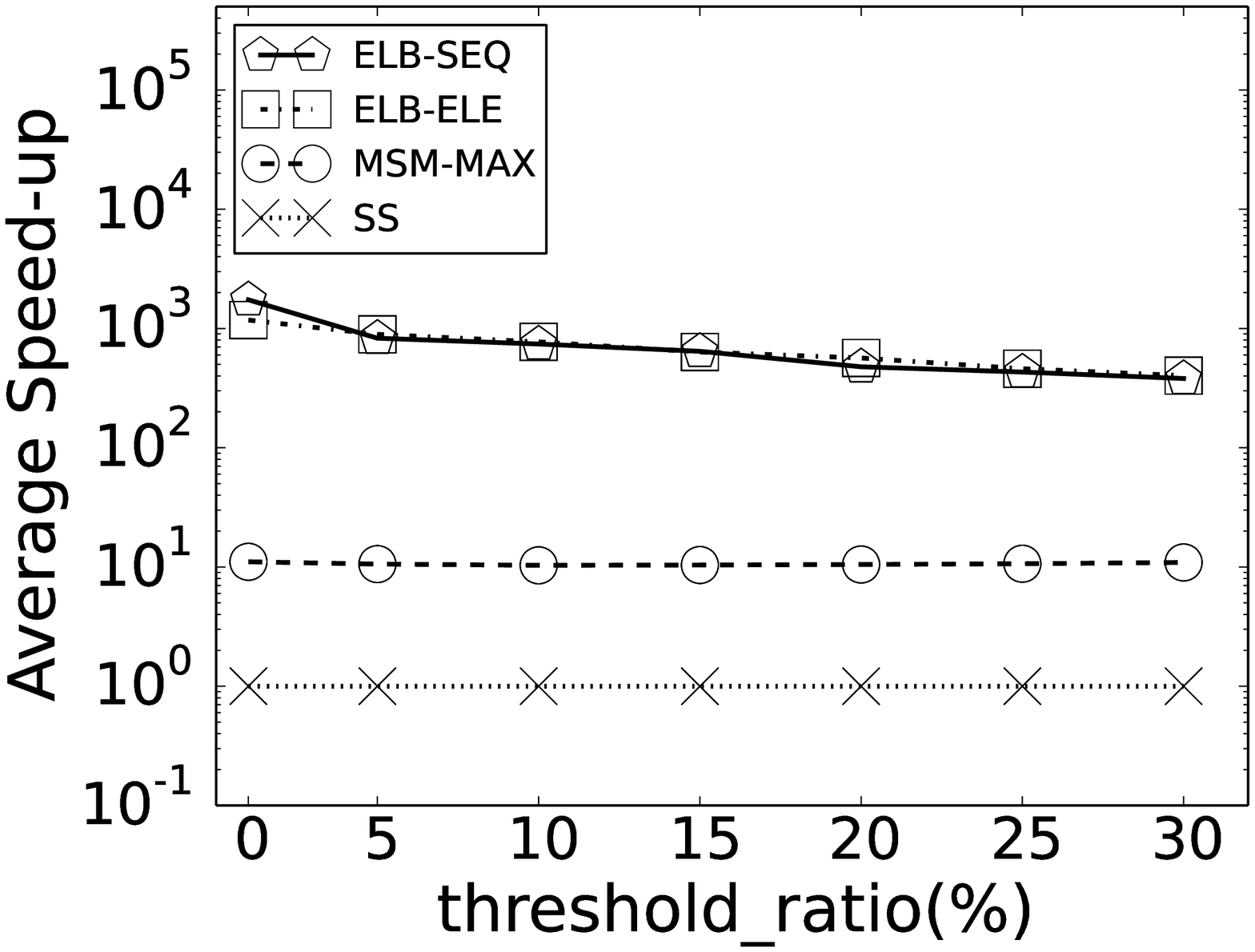}
		\caption[]%
		{UCR\_MALLAT}
	\end{subfigure}
	\caption{Speedup vs. \textit{threshold\_ratio} on synthetic datasets}
	\label{fig:4_threshold}
\end{figure}

The result on real-world datasets is shown in Figure~\ref{fig:5_distance}.  We show two representative figures and the others are consistent.
Our algorithms show better performance on different thresholds than others over all real-world datasets. 
Moreover, ELB-SEQ outperforms ELB-ELE clearly and remains stable when the threshold increases.
As shown in Figure~\ref{fig:5_distance}(b), the performance of ELB-ELE degrades when the threshold becomes larger. This is because the number of false candidates of ELB-ELE increases dramatically.
\begin{figure}[!htb]
	\centering
	\begin{subfigure}[b]{0.225\textwidth}
		\centering
		\includegraphics[width=1\textwidth]{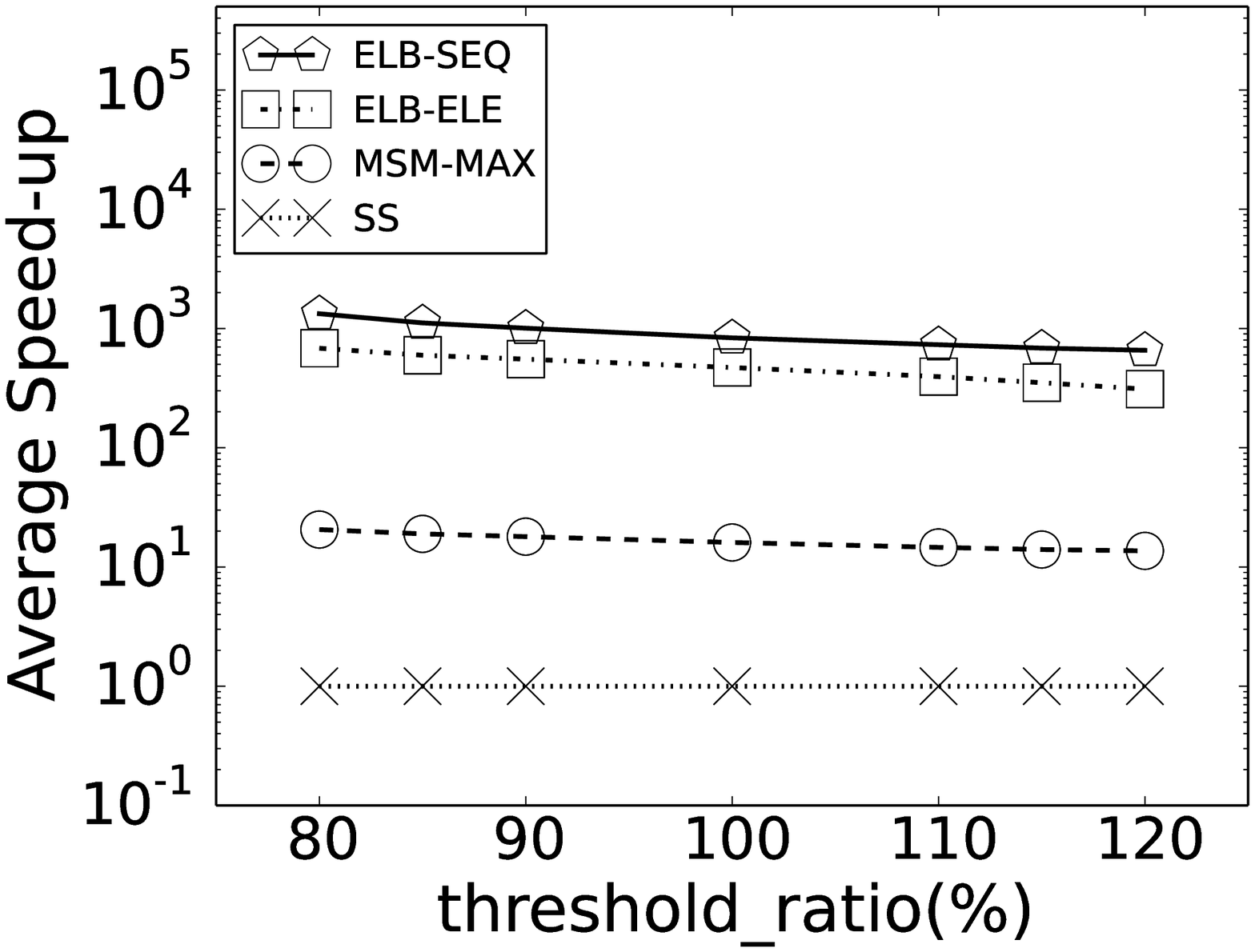}
		\caption[Network2]%
		{wind speed}    
	\end{subfigure}
	\hfill
	\begin{subfigure}[b]{0.225\textwidth}  
		\centering 
		\includegraphics[width=1\textwidth]{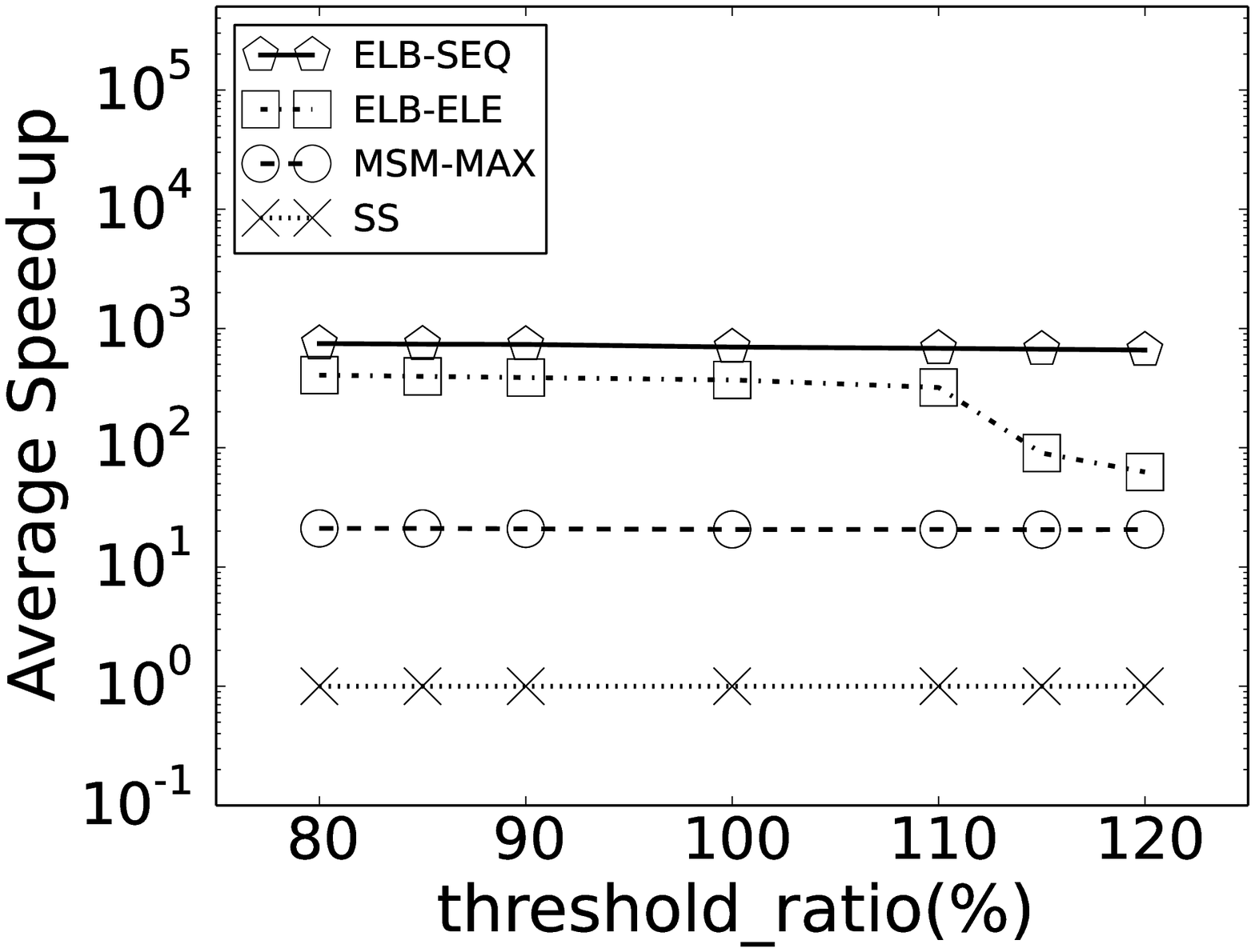}
		\caption[]%
		{generator speed}    
		\label{fig:5_distance-generator-speed}
	\end{subfigure}
	\caption{Speedup vs. threshold on real-world datasets}
	\label{fig:5_distance}
\end{figure}

According to implementations of $ ELB_{ele} $ and $ ELB_{seq} $, 
we know that the former is more efficient than the latter on calculating feature of window block.
In contrast, the bound of  $ ELB_{seq} $ is usually tighter than that of $ ELB_{ele} $.
Although ELB-ELE spends less time on the pruning phase than ELB-SEQ, 
it makes much more false candidates.
Comparing to synthetic datasets, patterns on real-world datasets have less deviation from corresponding streams,
which makes the pruning power of ELB-ELE degrade and the total cost increase on real-world datasets when the threshold gets larger.

\subsubsection{Impact of Pattern Occurrence Probability} 
\label{sub:comparision_of_matching_ratio}
In this section, we further examine the performance when varying
the pattern occurrence probability.
When the occurrence probability becomes lower, more windows will be filtered out in the pruning phase.
In contrast, when the probability becomes higher, more windows will enter the post-processing phase.
These two cases bring challenges to both pruning and post-processing phases.
A good approach should be robust to both situations.

We perform this experiment on synthetic datasets
varying the occurrence probability over 
$ \{10^{-3}, 5 \times 10^{-4}, 10^{-4}, 5\times 10^{-5},10^{-5} \}$.
Since the maximal length of four synthetic patterns is $ 1024 $ (in MALLAT dataset), 
the maximal testing probability is set to about $1/1024$.
Figure~\ref{fig:6_pattern_rate} illustrates the results. 
It can be seen that our algorithms still outperform MSM and SS in all examined probabilities. 
Furthermore, our algorithms show a larger speedup when the pattern occurrence probability become lower.
This experiment demonstrates the robustness of our algorithms over different occurrence probabilities.
\begin{figure}[!htb]
	\centering
	\begin{subfigure}[b]{0.225\textwidth}
		\centering
		\includegraphics[width=1\textwidth]{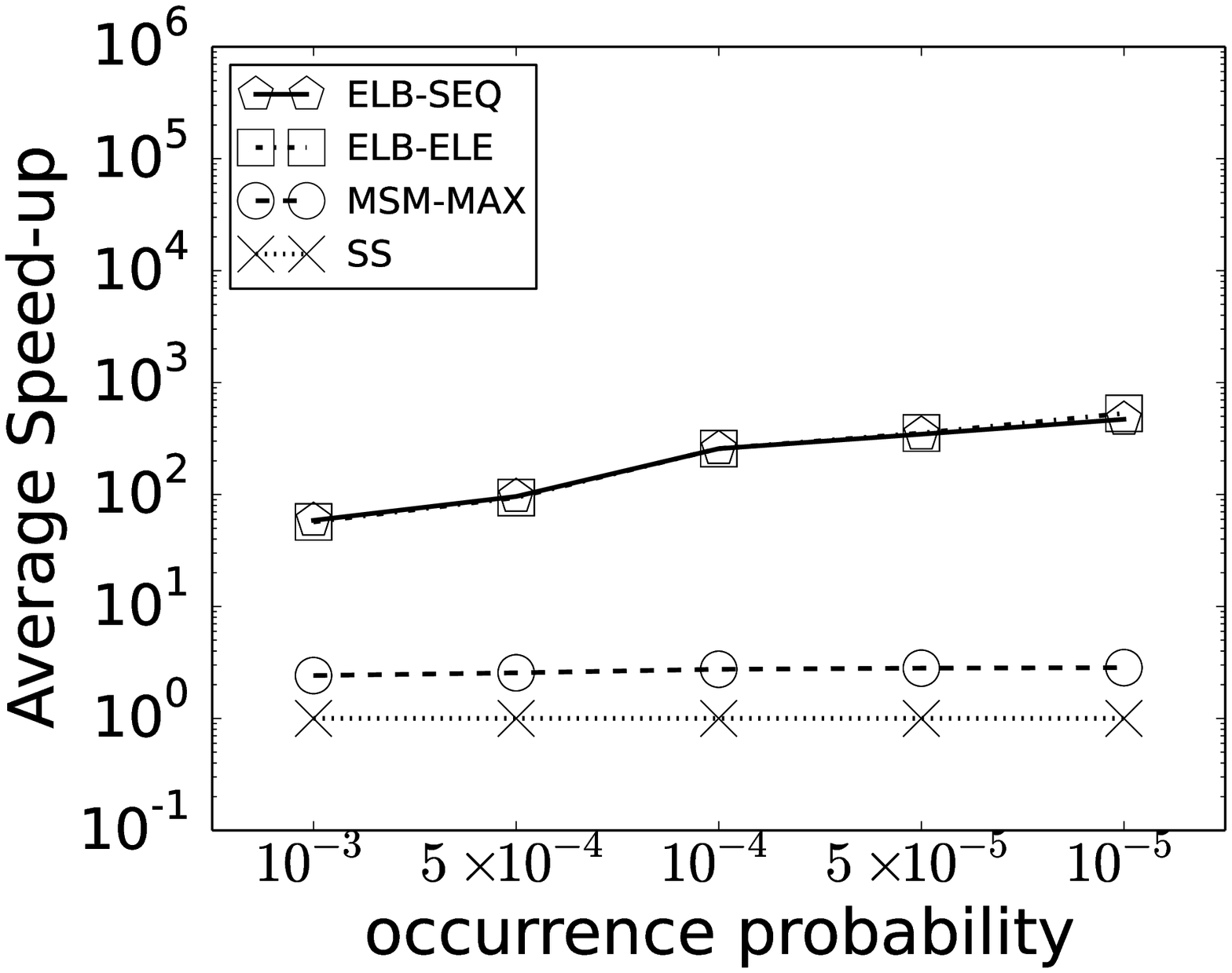}
		\caption[Network2]%
		{UCR\_Strawberry}    
	\end{subfigure}
	\hfill
	\begin{subfigure}[b]{0.225\textwidth}  
		\centering 
		\includegraphics[width=1\textwidth]{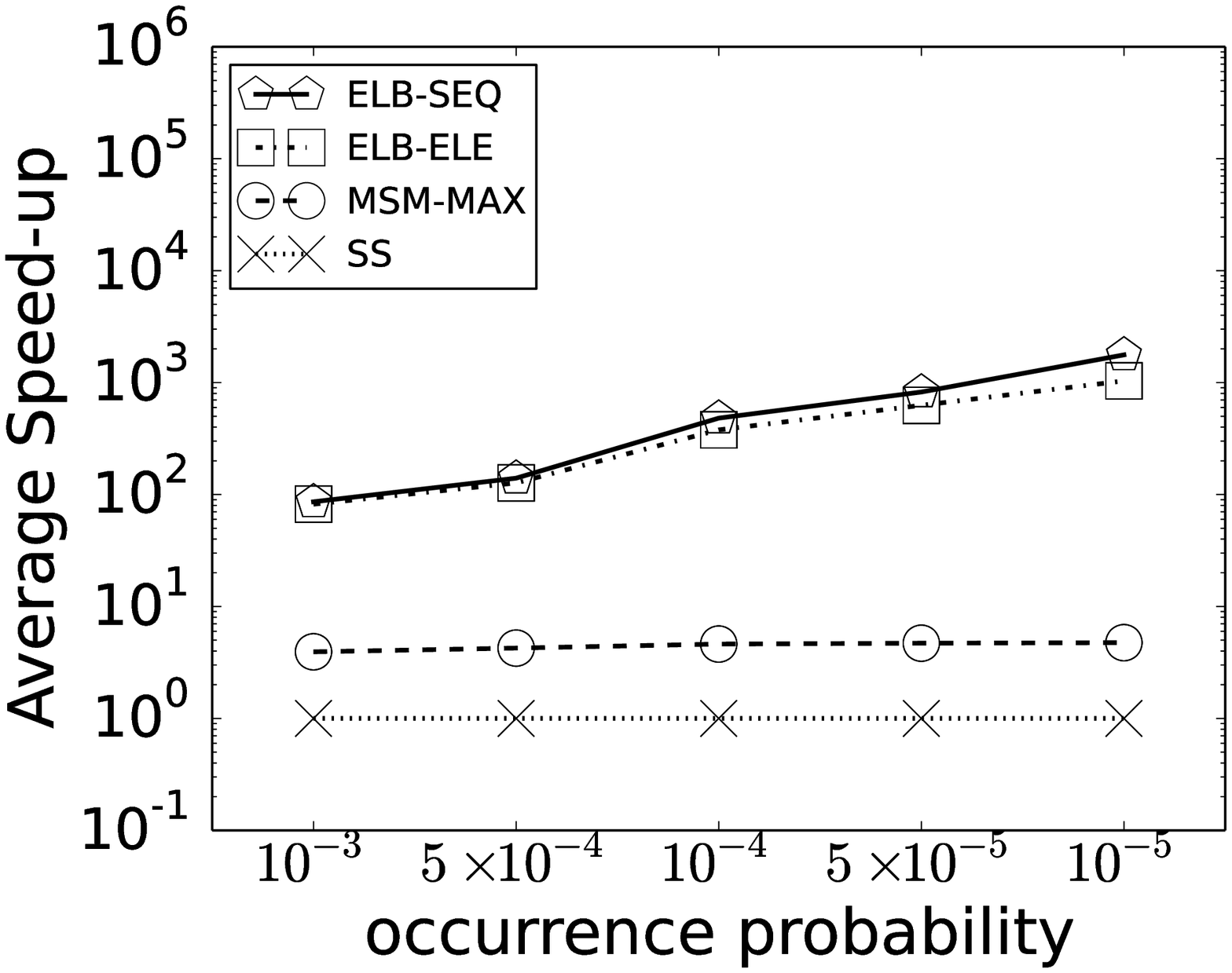}
		\caption[]%
		{UCR\_Meat}    
	\end{subfigure}
	\begin{subfigure}[b]{0.225\textwidth}   
		\centering 
		\includegraphics[width=1\textwidth]{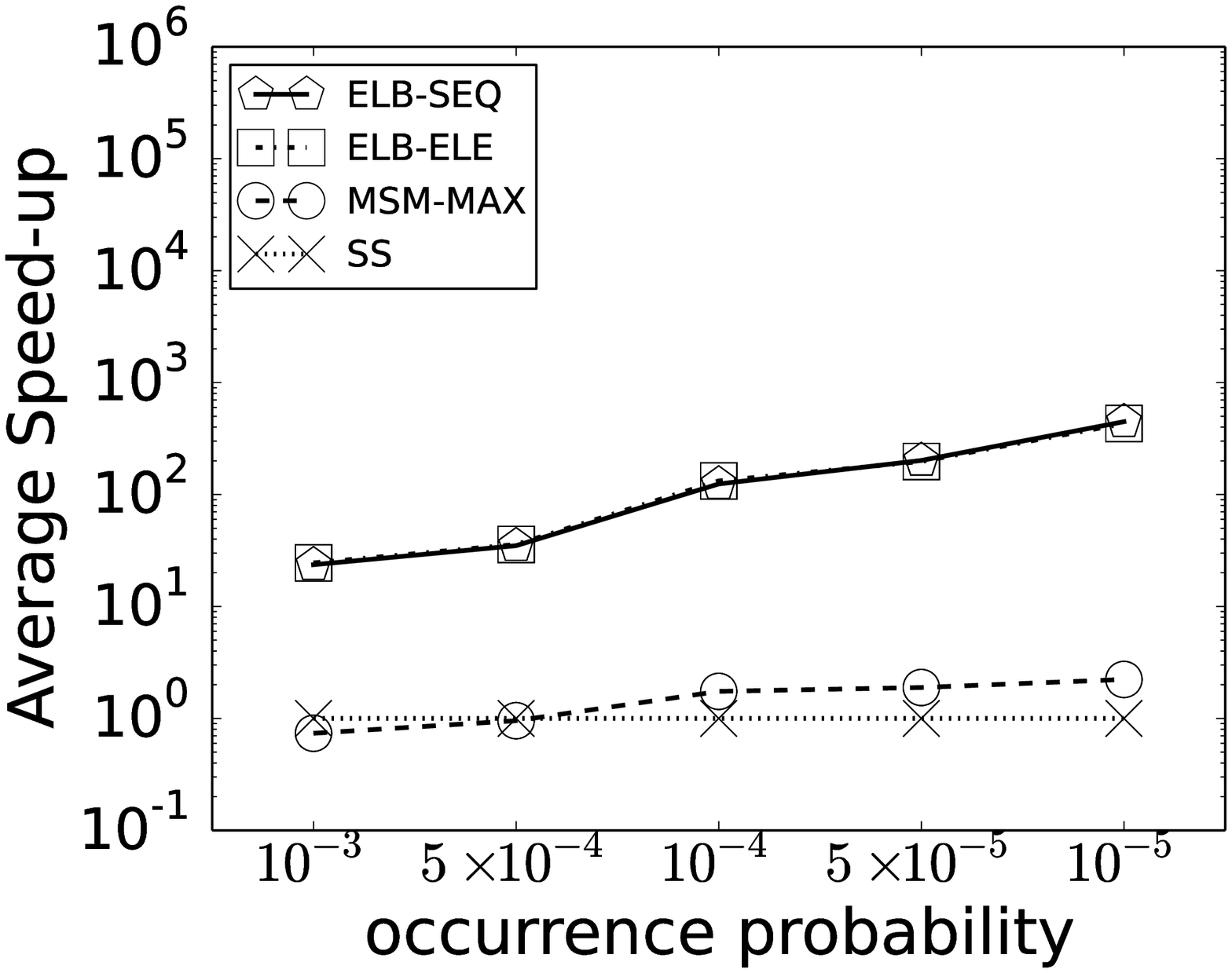}
		\caption[]%
		{UCR\_ECG}
	\end{subfigure}
	\hfill
	\begin{subfigure}[b]{0.225\textwidth}   
		\centering 
		\includegraphics[width=1\textwidth]{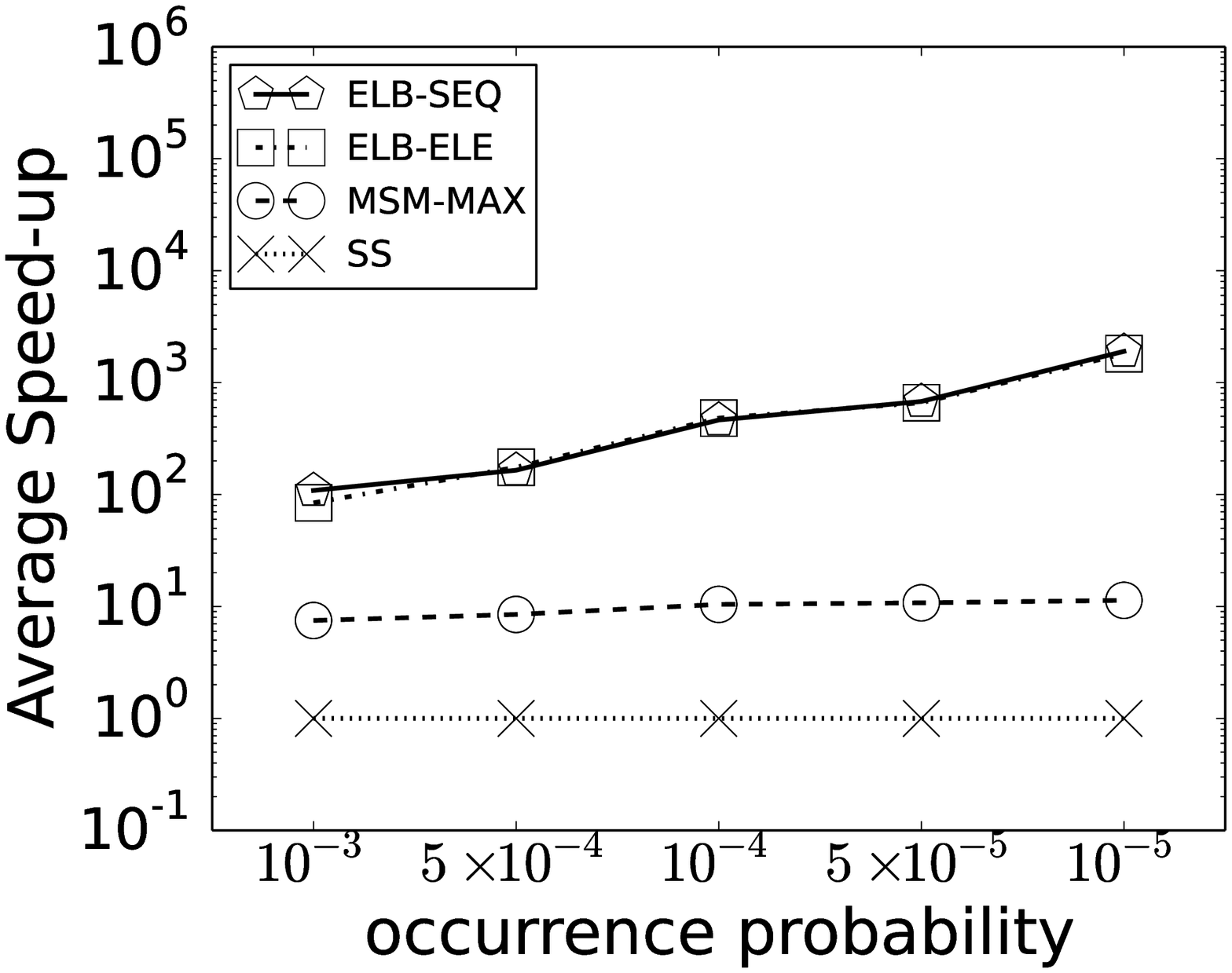}
		\caption[]%
		{UCR\_MALLAT}
	\end{subfigure}
	\caption{Speedup vs. pattern occurrence probability}
	\label{fig:6_pattern_rate}
\end{figure}

\subsubsection{Impact of Block Size} 
\label{sub:performance_block_size}
The block size is an important parameter affecting the pruning power of our approach.
In this experiment, we investigate the impact of block size by comparing ELB-ELE,  ELB-SEQ and MSM on both synthetic and real-world datasets.
We vary the ratio of the block size to the pattern length from $1\%$ to $40\%$.
A ratio being larger than $50\%$ indicates that the entire pattern contains only one block, which makes BSP meaningless.

Figure~\ref{fig:1_window_block} shows the experimental results on some representative synthetic and real-world datasets. The rest are consistent. 
A too small or too large block size results in performance degradation.
In detail, 
a smaller block size leads to a tighter bound for each block which improves the pruning effectiveness.
Nevertheless,  a small block size, corresponding to a small sliding step, results in more block computation and 
higher cost in the pruning phase.
A larger block size may bring less block computation but a looser bound. 
The loose bound incurs degradation of the pruning effectiveness.
In practice, our algorithms achieves the optimal performance when the block size is about $5\%$ to $10\%$ of the pattern length.
\begin{figure}[!htb]
	\centering
	\begin{subfigure}[b]{0.225\textwidth}
		\centering
		\includegraphics[width=1\textwidth]{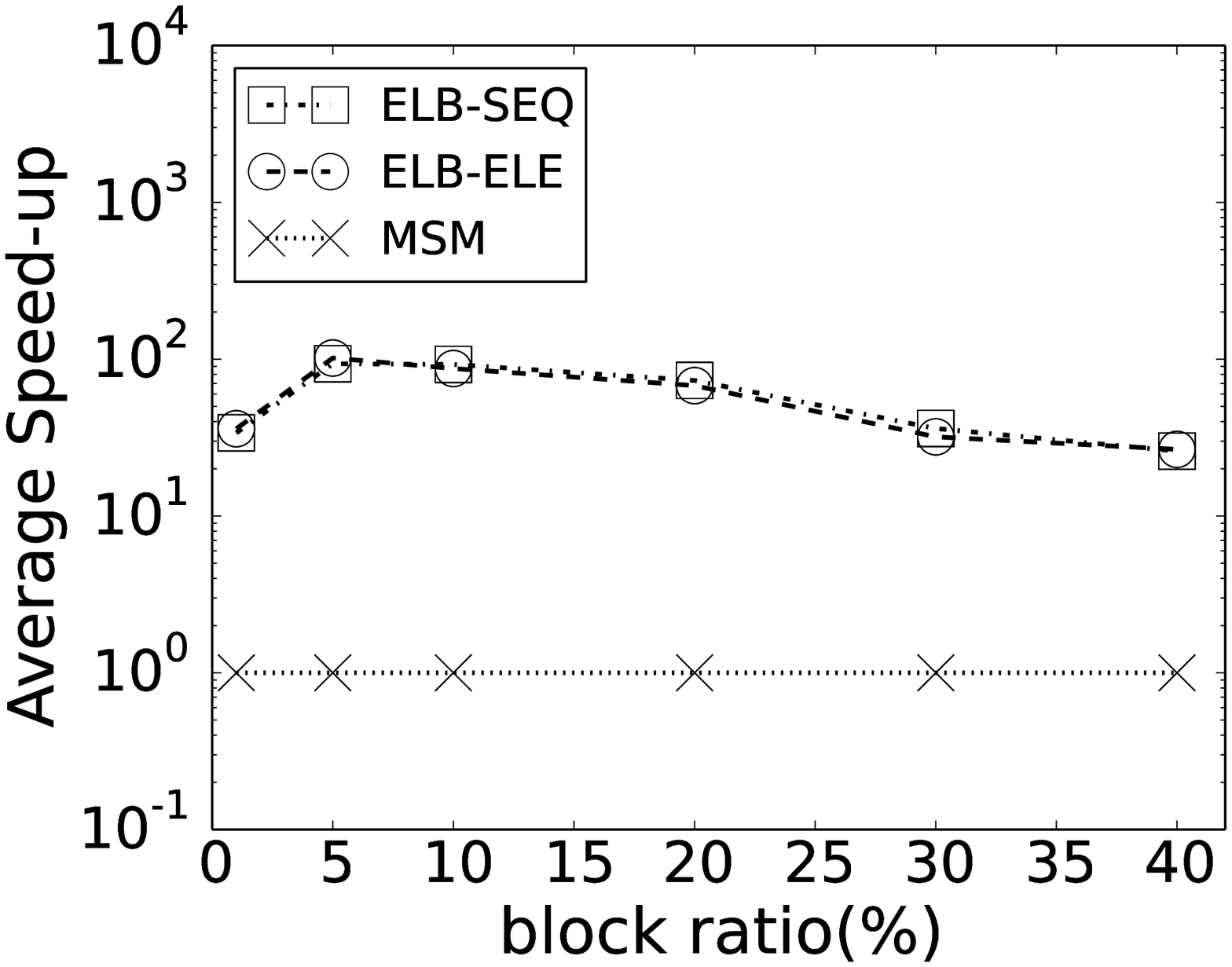}
		\caption[]%
		{ UCR\_Strawberry}
	\end{subfigure}
	\hfill
	\begin{subfigure}[b]{0.225\textwidth}
		\centering
		\includegraphics[width=1\textwidth]{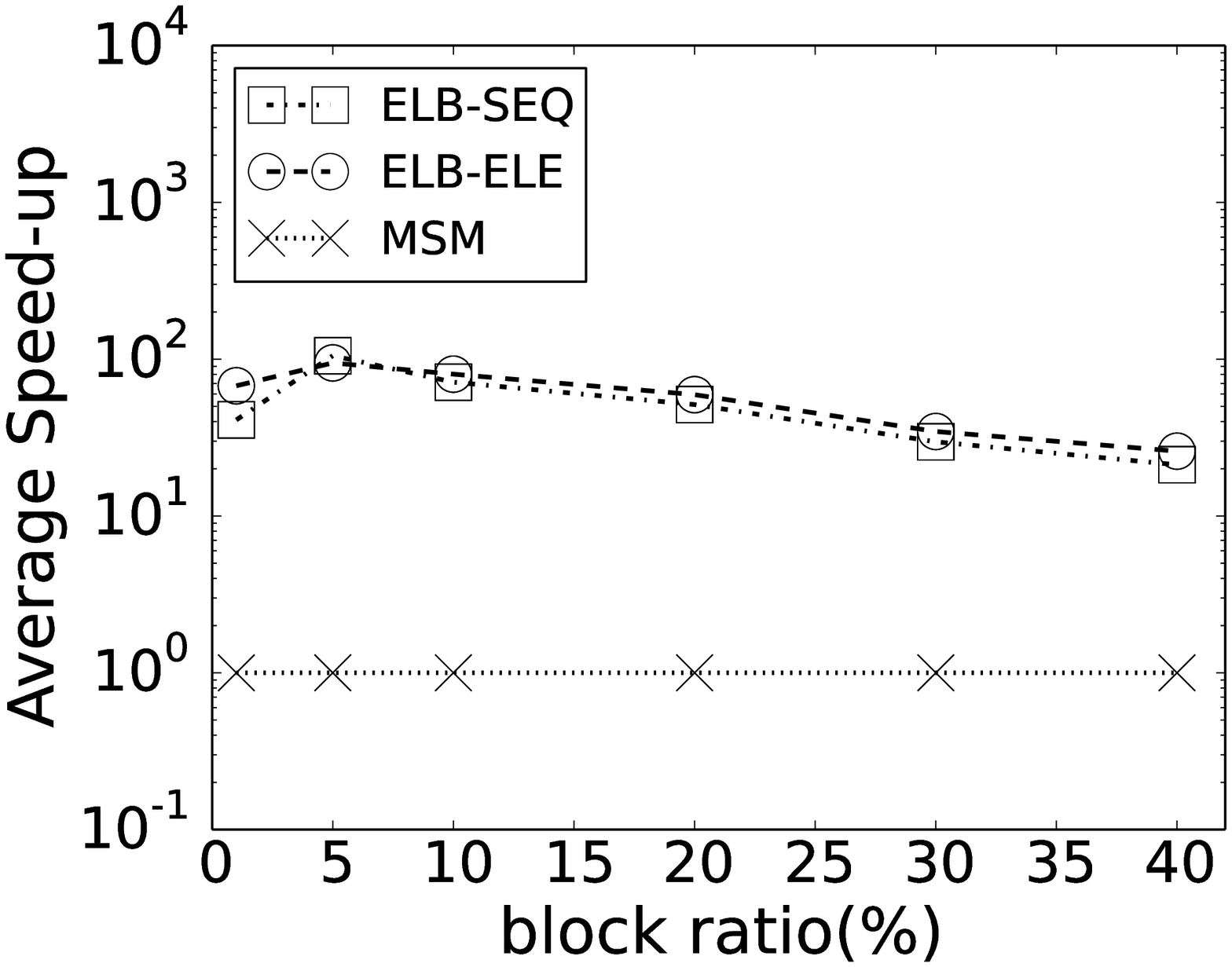}
		\caption[]%
		{ UCR\_Meat}
	\end{subfigure}
	\begin{subfigure}[b]{0.225\textwidth}   
		\centering 
		\includegraphics[width=1\textwidth]{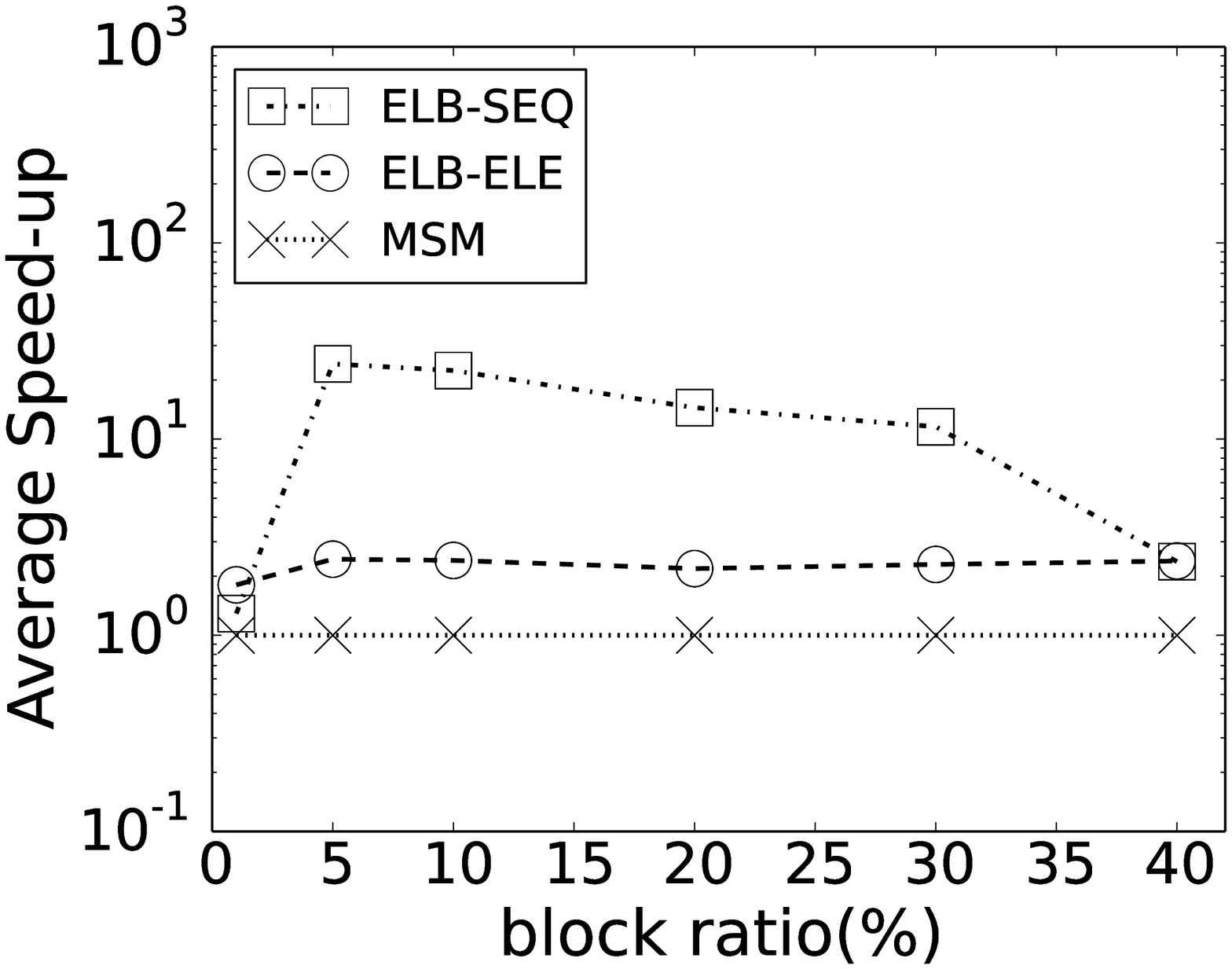}
		\caption[]%
		{ generator speed
	}
	\end{subfigure}
	\hfill
	\begin{subfigure}[b]{0.225\textwidth}   
		\centering 
		\includegraphics[width=1\textwidth]{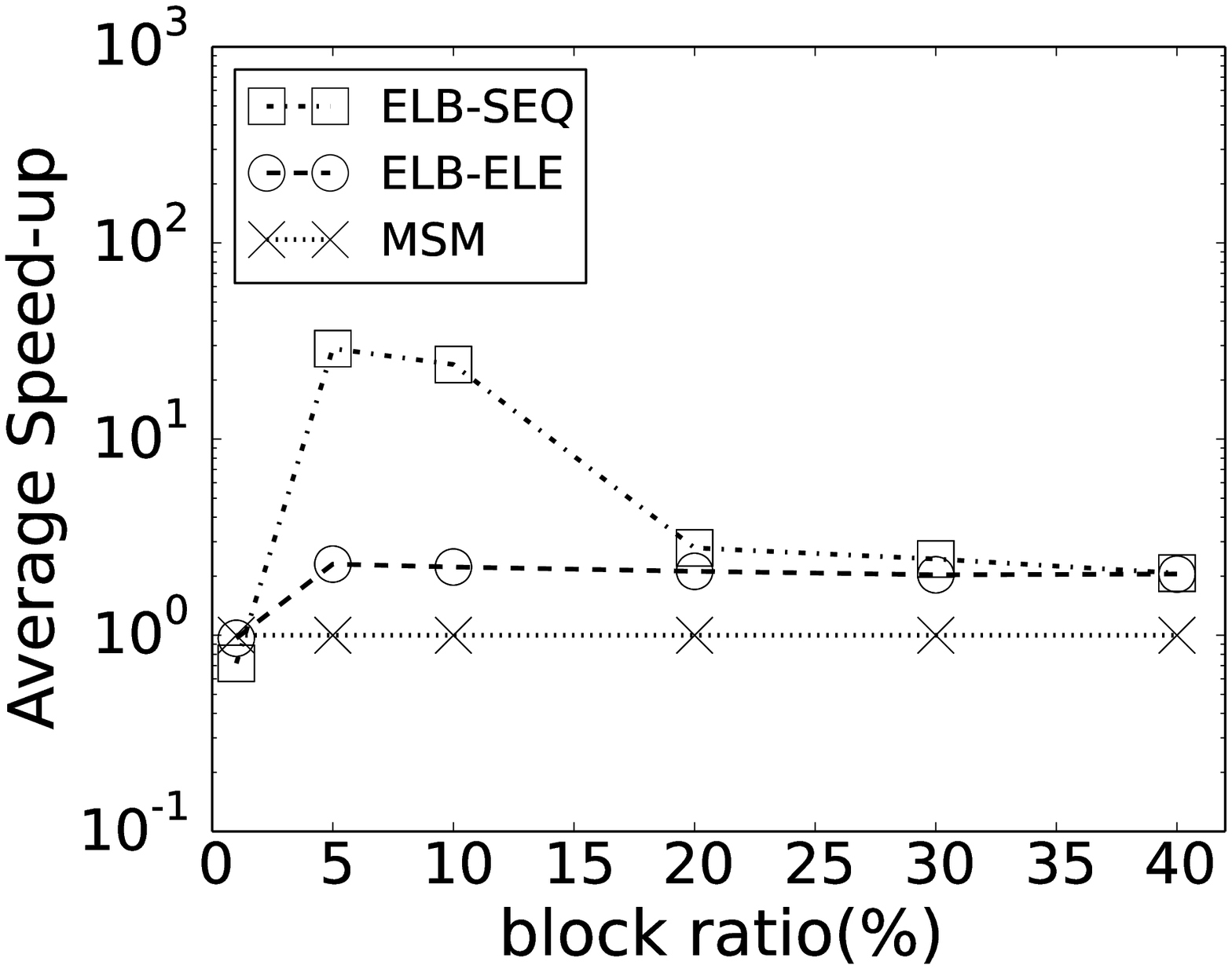}
		\caption[]%
		{ converter power}
	\end{subfigure}
	\caption{Speedup vs. \textit{block\_ratio}.}
	\label{fig:1_window_block}
\end{figure}

\subsubsection{Impact of BSP Optimization} 
\label{sub:bsp_investigation}
In this section, we investigate the effectiveness of the BSP policy.
BSP policy only accelerates the pruning phase. 
In terms of the cost of the pruning phase, 
the algorithm with BSP outperforms that without BSP
on all synthetic datasets and $ 80\% $ real-world datasets.
On synthetic datasets, the benefit of BSP is significant since the cost of the pruning phase is comparable with the cost of the post-processing phase.
On real-world datasets, however, the benefit of BSP is relatively small because the cost of the pruning phase is much smaller than that of the post-processing phase.
Considering the total cost, the algorithm with BSP policy provides an average of $ 27.44\% $ speedup on synthetic datasets, yet an average of $ 1.16\% $ speedup on real-world datasets.

	\section{Related Work}
\label{sec:relate_work}
There are two categories of the related works, multiple patterns matching over streaming  time series and subsequence similarity search.

\noindent  \textbf{Multiple patterns matching over streaming time series }.
Traditional single pattern matching over stream is relatively trivial, hence recent research works put more focus on optimizing the multiple pattern scenario. 
Atomic wedge \cite{wei_atomic_2005} is proposed to monitor stream with a set of pre-defined patterns, which exploits the commonality among patterns.
Sun et al. \cite{sun_matching_2009} extends atomic wedge for various length queries and tolerances.
Lian et al. \cite{lian_similarity_2007} propose a multi-scale segment mean (MSM) representation to detect static patterns over streaming time series, and discuss the batch processing optimization and the case of dynamic patterns in its following work \cite{lian_multiscale_2009}.
Lim et al. \cite{lim_similar_2008} propose SSM-IS which divides long sequences into smaller windows.
Although these techniques are proposed for streaming time series and some of them speed up the distance calculation between the pattern and the candidate, most of them focus on exploring the commonality and correlation among multiple patterns for pruning unmatched pattern candidates, which doesn't aim to reduce the complexity brought by fine-grained single pattern matching problem.

\noindent  \textbf{Subsequence similarity search}.
FRM \cite{faloutsos_fast_1994} is the first work for subsequence similarity search which maps data sequences in database into multidimensional rectangles in feature space.
General Match \cite{moon_general_2002} divides data sequences into generalized sliding windows and the query sequence into generalized disjoint windows, which focuses on estimating parameters to minimize the page access.
Loh et al. \cite{loh_subsequence_2004} propose a subsequence matching algorithm that supports normalization transform.
Lim et al. \cite{lim_using_2006} address this problem by selecting the most appropriate index from multiple indexes built on different windows sizes.
Kotsifakos et al. \cite{kotsifakos_subsequence_2011} propose a framework which allows gaps and variable tolerances in query and candidates.
Wang et al. \cite{wang_data-adaptive_2013} propose DSTree which is a data adaptive and dynamic segmentation index on time series.
This category of researches focuses on indexing the common features of \textit{archived} time series, which is not optimized for pattern matching over stream.

	\section{Conclusions}\label{sec:conclusion}
In this paper, we formulate a new problem, called ``fine-grained pattern matching'', 
which allows users to specify varied granularities of matching deviation for different segments of a given pattern, and fuzzy regions for adaptive breakpoint determination between consecutive segments. 
We propose ELB representation together with BSP optimization to prune sliding windows efficiently, and an adaptive post-processing algorithm which enables us to determine breakpoints in linear complexity. 
We conduct extensive experiments on both synthetic and real-world datasets to illustrate that our algorithm outperforms the baseline solution and prior-arts. 
	
	\balance
	
	\bibliographystyle{abbrv} 
	\bibliography{pattern-m-v0.6}
	
	\begin{appendix}

\section{Proof of Subsequence-based ELB Representation}
\label{proof:seq}
Now We prove that $ ELB_{seq} $ satisfies the lower bounding property as follows:
\begin{proof}
	Given $P$ and $\hat{W}_t$, suppose that $W_{t,i}$ is a fine-grained matching of $P$ where $0 \leqslant i < w$. 
	Let us consider a window block $\hat{W}_{t,j}$ for any $j\in [1,N]$ 
	which is composed of $W_t[j\cdot w - w + 1 : j\cdot w]$
	and aligns with a subsequence  $P[j\cdot w - w + 1-i:j\cdot w - i]$, denoted by $sp_{j\cdot w - i}$ for brevity.
	
	According to Equation~\ref{eq:seq_limit}, there are:
	\begin{displaymath}
	|\mu_{t,j} - \mu_{sp_{j\cdot w-i}}| \leqslant \theta_{seq}(j\cdot w-i)
	\end{displaymath}
	Remove the absolute value symbol:
	\begin{displaymath}
	- \theta_{seq}(j\cdot w-i) \leqslant 
	\mu_{t,j} - \mu_{sp_{j\cdot w-i}}
	 \leqslant \theta_{seq}(j\cdot w-i)
	\end{displaymath}
	Focus on the right inequality, we deduce that:
	\begin{displaymath}
	\mu_{t,j} \leqslant 
	\mu_{sp_{j\cdot w-i}}
	+ \theta_{seq}(j\cdot w-i)
	\end{displaymath}
	Considering the definition of $ U $ (Equation~\ref{eq:avg_ul})
	and $ \hat{P}_j^u $ (Equation~\ref{eq:avg_bound}),
	it holds that:
	\begin{displaymath}
	\mu_{t,j} \leqslant U_{j\cdot w-i}
	\leqslant
	\max \limits_{0 \leqslant a < w} (U_{j\cdot w - a}) = \hat{P}_j^u
	\end{displaymath}
	Similarly we also prove $\mu_{t,j} \geqslant \hat{P}_j^l$. 
	
	The proof is completed.
\end{proof}

\section{Proof of Pattern Transformation}
\label{proof:scheme-3}
To use MSM in our scenario, we adopt a pattern transformation which discards the fuzzy parts of $ P $ and sets the threshold of a segment to be its maximal allowable value. We give a brief proof as follows:
\begin{proof}
	If $ C $ is a fine-grained matching of $ P $, for each $ C_i $, there is:
	\begin{displaymath}
	ED_{norm}(C_k, P_k) \leqslant \varepsilon_k
	\end{displaymath}
	According to Definition~\ref{def:ld_p}, we can deduce that:
	\begin{equation*}
	\begin{split}
	&\frac{1}{bp_{k} - (bp_{k-1}+1)}(\sum^{bp_{k}}_{i=bp_{k-1}+1} |{c_i-p_i}|^2) \leqslant \varepsilon_k^2\\
	&\sum^{l_{k}}_{i=r_{k-1}+1} |{c_i-p_i}|^2 \leqslant \sum^{bp_{k}}_{i=bp_{k-1}+1} |{c_i-p_i}|^2 \leqslant (r_{k}-l_{k-1})\varepsilon_k^2\\
	&ED_{norm}(C'_k, P'_k) \leqslant
	\varepsilon'_k =  \varepsilon_k \sqrt{\frac{r_{k}-l_{k-1}}{l_{k}-r_{k-1}}} 
	= \varepsilon'_k
	\end{split}
	\end{equation*}
	It indicates that pruning phase based on the transformed pattern doesn't cause any false dismissals.
	
	The proof is completed.
\end{proof}

\end{appendix}
\end{document}